\documentclass[10pt]{article}
\usepackage[compress, round]{natbib}
\usepackage{ss}
\usepackage[utf8]{inputenc}
\usepackage[T1]{fontenc}

\usepackage[verbose=true,
letterpaper,
textheight=8.7in,
textwidth=6.4in,
headheight=12pt,
headsep=25pt,
footskip=30pt]{geometry}

\usepackage{mathtools}
\usepackage{booktabs}
\usepackage{graphicx}
\usepackage{amsmath,amssymb,amsfonts,amsthm,amsxtra,bm}
\usepackage{amsthm}
\usepackage{xspace}
\usepackage{enumerate}
\usepackage{xr-hyper}
\usepackage{hyperref}
\usepackage{url}
\usepackage{colortbl}
\usepackage{makecell,multirow}       % professional-quality tables
\usepackage{nicefrac}       % compact symbols for 1/2, etc.
\usepackage{thmtools}  
\usepackage{xspace}
\usepackage{footnote, tablefootnote}
\usepackage{float}
\floatstyle{plaintop}
\restylefloat{table}
\usepackage{epstopdf}
\usepackage{latexsym}
\usepackage{algorithm, algorithmicx, algpseudocode}
\usepackage{todonotes}
\usepackage{thmtools, thm-restate}
\usepackage{bbm}

\numberwithin{equation}{section}
\definecolor{darkblue}{rgb}{0.0,0.0,0.65}
\definecolor{darkred}{rgb}{0.68,0.05,0.0}
\definecolor{darkgreen}{rgb}{0.0,0.29,0.29}
\definecolor{darkpurple}{rgb}{0.47,0.09,0.29}

\hypersetup{
   colorlinks = true,
   citecolor  = darkblue,
   linkcolor  = darkred,
   filecolor  = darkblue,
   urlcolor   = darkblue,
 }

\newcommand{\BVI}{\textsc{MOMA}}
\newcommand{\VIH}{\textsc{VI-Hoeffding}}
\newcommand{\VIB}{\textsc{VI-Bernstein}}
\newcommand{\plan}{\textsc{Planning}}
\newcommand{\DU}{\textsc{Dual-Update}}
\newcommand{\PDU}{\textsc{Projection-based-Dual-Update}}
\newcommand{\ODU}{\textsc{Projection-free-Dual-Update}}
\newcommand{\DODU}{\textsc{Double-Dual-Update}}

\DeclareMathOperator*{\argmax}{\arg\max}
\DeclareMathOperator*{\argmin}{\arg\min}

\newcommand{\dist}{\mathrm{dist}}

\newcommand{\vtheta}{\bm{\theta}}
\newcommand{\vx}{\mathbf{x}}
\newcommand{\vr}{\mathbf{r}}
\newcommand{\vV}{\mathbf{V}}
\newcommand{\vQ}{\mathbf{Q}}
\newcommand{\vW}{\mathbf{W}}

\newcommand{\paren}[1]{{\left( #1 \right)}}

\newcommand{\set}[1]{{\left\{ #1 \right\}}}

\newcommand{\mat}[1]{\ensuremath{\mathbf{#1}}}

\newcommand{\trans}{^{\top}}

\newcommand{\abs}[1]{|{#1}|}

\newcommand{\E}{\mathbb{E}}
\newcommand{\W}{\mathbb{W}}
\newcommand{\D}{\mathbb{D}}
\renewcommand{\P}{\mathbb{P}}
\newcommand{\Var}{\text{Var}}

\newcommand{\Phat}{\widehat{\P}}

\newcommand{\eps}{\epsilon}

\newcommand{\cO}{\mathcal{O}}
\newcommand{\tlO}{\mathcal{\tilde{O}}}

\newcommand{\R}{\mathbb{R}}

\newcommand{\B}{\mathbb{B}}

\newcommand{\V}{\mathbb{V}}
\newcommand{\X}{\mathbb{X}}
\newcommand{\Y}{\mat{Y}}

\newcommand{\cA}{\mathcal{A}}
\newcommand{\cB}{\mathcal{B}}

\newcommand{\cF}{\mathcal{F}}

\newcommand{\cL}{\mathcal{L}}

\newcommand{\cS}{\mathcal{S}}
\newcommand{\cV}{\mathcal{V}}
\newcommand{\cW}{\mathcal{W}}

\newcommand{\setto}{\leftarrow}
\newcommand{\up}[1]{\overline{#1}}
\newcommand{\low}[1]{\underline{#1}}

\newcommand{\NASH}{\textsc{Nash}}

\newcommand{\mg}{\mathrm{MG}}

\newcommand{\nlsum}{\sum\nolimits}

\usepackage{subfigure}

\newtheorem{theorem}{Theorem}
\newtheorem{lemma}[theorem]{Lemma}

\newtheorem{assumption}{Assumption}

\begin{document}

\title{Provably Efficient Algorithms for 
Multi-Objective Competitive RL}

\author{\name Tiancheng Yu \email{yutc@mit.edu}\\
   \addr{Massachusetts Institute of Technology}\\
   \name Yi Tian \email{yitian@mit.edu}\\
   \addr{Massachusetts Institute of Technology}\\
   \name Jingzhao Zhang \email{jzhzhang@mit.edu}\\
   \addr{Massachusetts Institute of Technology}\\
      \name Suvrit Sra \email{suvrit@mit.edu}\\
\addr{Massachusetts Institute of Technology}\\
}

\maketitle

\begin{abstract}
    We study multi-objective reinforcement learning (RL) where an agent's reward is represented as a vector. In settings where an  agent competes against opponents, its performance is measured by the distance of its average return vector to a target set. We develop statistically and computationally efficient algorithms to approach the associated target set. Our results extend Blackwell's approachability theorem~\citep{blackwell1956analog} to tabular RL, where strategic exploration becomes essential. The algorithms presented are adaptive; their guarantees hold even without Blackwell's approachability condition. If the opponents use fixed policies, we give an improved rate of approaching the target set while also tackling the more ambitious goal of simultaneously minimizing a scalar cost function. We discuss our analysis for this special case by relating our results to previous works on constrained RL. To our knowledge, this work provides the first provably efficient algorithms for vector-valued Markov games and our theoretical guarantees are near-optimal.
\end{abstract}

\section{Introduction}
What can a player expect to achieve in competitive games when pursuing multiple objectives? If the player has a single objective, the answer is clear from von Neumann's minimax theorem~\citep{neumann1928theorie}: the player can follow a fixed strategy to ensure that its cost is no worse than a certain threshold, the \emph{minimax value} of the game, no matter how the opponents play. But if the player has multiple objectives, the answer is less clear and it must define some tradeoffs. One important way to capture tradeoffs is to define a certain \emph{target set} of vectors, and then to ensure that player's vector of returns lies in this set. The player's performance can then be measured via the distance of its reward vector from the target set. In 1956, Blackwell showed that in a repeated game, the player of interest can make the distance of its average return to a target set small as long as this set satisfies a condition called \emph{approachability}~\citep{blackwell1956analog}. 

The approachability theorem applies to multi-objective games with a decision horizon of a \emph{single} time step. However, in many practical domains such as robotics, self-driving, video games, and recommendation systems, the decision horizons span \emph{multiple} time steps. For example, in a robot control task, we may hope the robot arm reaches a certain region in a 3D space; while, in self-driving, we may hope the car takes care of speed, safety and comfort simultaneously. In these problems, the state of the decision process transitions based on both the actions taken by the players and the unknown dynamics. Though a generalization (Assumption~\ref{ass:Approachability}) of Blackwell's approachability condition~\citet{blackwell1956analog} is relatively direct, efficient exploration and the need to learn the unknown transitions is what poses a challenge in the multiple time step setting.

This challenge motivates us to ask: \emph{How can a player approach a target set that satisfies a generalized notion of approachability?} We answer this question by modeling multi-objective competitive reinforcement learning (RL) as an online learning problem in a vector-valued Markov game (MG), for which we provide efficient algorithms as instances of a generic meta-algorithm that we propose. 

Going one step further, we can ask a more ambitious question: \emph{Can we minimize a scalar cost function while also satisfying approachability?} Our answer is affirmative if the opponents play fixed policies; equivalently, if the agent interacts with a fixed environment (without opponents), in which case the model reduces to a vector-valued Markov decision process (MDP). In this setting, the target set can be viewed as a set of constraints, and our results improve on the rich literature on constrained MDP in multiple aspects.

In Table~\ref{tab:placement} we give a comparison of different multi-objective RL settings. Our work can be seen as a generalization of both~\citep{blackwell1956analog} and~\citep{agrawal2014bandits} to cases with an $H$ step horizon. 

\begin{table}[t]
  \centering
  \caption{The settings of this work with reference to the literature} \label{tab:placement}
  \begin{tabular}{ccc}
    \toprule 
    & w/o opponents & w/ adversarial opponents \\ 
    \midrule
    single-state single-horizon & \makecell{constrained bandits \\ (e.g.,~\citep{agrawal2014bandits})} & \makecell{vector-valued games \\ (e.g.,~\citep{blackwell1956analog})} \\
    \midrule 
    multi-state $H$-horizon & \makecell{constrained MDPs \\ (e.g.,~\citep{brantley2020constrained}; \textbf{this work})} & \makecell{vector-valued Markov games \\ (\textbf{this work})} \\
    \bottomrule
  \end{tabular}
\end{table}

\noindent\textbf{Summary of our contributions.}
\begin{list}{{\tiny$\blacktriangleright$}}{\leftmargin=1em}
  \setlength{\itemsep}{-1pt}

  \item For online learning in vector-valued Markov games, we propose two provably efficient algorithms to approach a target set under a generic framework. Strategic exploration is essential to obtain statistical efficiency (Theorems~\ref{thm:basic-approachable} and~\ref{thm:OCO-delta-approachable}) for both algorithms. The second algorithm has the merit of being more computationally efficient.
  \item When the chosen target set is not approachable, both our algorithms adapt automatically. Concretely, we describe the guarantees (Theorems~\ref{thm:basic-delta-approachable} and~\ref{thm:OCO-delta-approachable}) of the algorithms using a notion of $\delta$-approachability (Assumption~\ref{ass:delta-Approachability}).
  \item For vector-valued MDPs, via a more dedicated design of the exploration bonus, we obtain a near-optimal rate of making the average reward vector approach (Theorem~\ref{thm:bernstein}) the target set. Moreover, under a mild assumption, we present a modified algorithm that can simultaneously minimize a convex cost function (Theorem~\ref{thm:satisfiable}). Comparing with existing results in constrained MDP, our bounds on regret and constraint violation are the sharpest with respect to their dependence on the parameters $S$, $A$, and $K$, where $S$ is the number of states, $A$ is the number of actions and $K$ is the number of episodes.
\end{list}

\subsection{Related Work}
\noindent\textbf{Blackwell's approachability.}
\citet{blackwell1956analog} initiated the study of multi-objective learning in repeated matrix games by introducing the notion of approachability and an algorithm to approach a given set. Using a dual formulation of the distance from a point to a convex cone, \citet{abernethy2011blackwell} show the equivalence of approachability problems and online linear optimization. \citet{shimkin2016online} further extends the equivalence to online convex optimization (OCO) via a dual formulation of the distance from a point to a convex set. Our primal and dual algorithms generalize respectively Blackwell's algorithm~\citep{blackwell1956analog} and the OCO-based algorithm~\citep{shimkin2016online} to Markov games.

\textbf{Learning in Markov games.}
Markov games, also known as stochastic games~\citep{shapley1953stochastic, littman1994markov}, are a general model for multi-agent reinforcement learning. In recent years, much attention has been given to learning in scalar-valued Markov games with unknown transitions. In the self-play setting~\citep{bai2020provable, xie2020learning, bai2020near, liu2020sharp}, the goal is to learn a Nash equlibrium with sample complexity guarantees. \citet{bai2020provable, xie2020learning, bai2020near} consider zero-sum Markov games while \citet{liu2020sharp} provide results for general-sum Markov games.
In the online setting~\citep{brafman2002r, xie2020learning, tian2020provably}, the goal is to achieve low regret in presence of an adversarial opponent. We also study the online setting, but in contrast, we consider vector-valued returns and the goal is to make the average return approach a given set.

\textbf{Online learning with constraints.}
Multi-objective RL is closely related to RL with constraints since satisfying the constraints is tantamount to having extra objectives. 
\citet{badanidiyuru2013bandits} study bandits with knapsacks, and \citet{agrawal2014bandits} study the more general setting with concave rewards and convex constraints that the method needs to approach. Beyond bandits, \citet{jenatton2016adaptive,yuan2018online} study online convex optimization with constraints given by convex functions.

\textbf{Constrained MDPs.}
For MDPs with linear constraints, \citet{efroni2020exploration, ding2020provably, qiu2020upper, brantley2020constrained} provide algorithms with both regret and total constraint violation guarantees. As a generalization of~\citep{agrawal2014bandits}, \citet{brantley2020constrained} also consider MDPs with convex constraints and concave rewards and discuss as a special case MDPs with knapsacks on \emph{all} episodes. \citet{chen2020efficient} formulate MDPs with knapsacks on \emph{each} episode as factored MDPs, to which the regret bounds of factored MDPs~\citep{osband2014near, tian2020towards, chen2020efficient} apply. See the discussion at the end of Section~\ref{sec:satisfiable} for a detailed comparison.

\textbf{Multi-objective RL with preference.}
More recently, \citet{wu2020accommodating} study single-agent multi-objective RL to accommodate potentially adversarial preference vectors. 
In contrast, we assume a potentially adversarial opponent that influences both the transition and the return vector.  Their goal also differs from ours in that they aim to maximize the cumulative rewards defined by the observed preference vectors in each episode. The preference vector in their setting is similar to the dual variable in our algorithm. Nonetheless, our dual variable is learned by an update procedure.

All of the aforementioned works on MGs or MDPs focus on the episodic setting. See, e.g.,~\citep{cheung2019non, singh2020learning}, for the studies of multi-objective or constrained RL in the nonepisodic setting.

\section{Background and Problem Setup}

In this section, we formulate the problem of two-player zero-sum Markov Games. We  control one of the players, whom we call the \emph{agent}. The other player is referred to as the \emph{adversary}. We use the two-player zero-sum condition for simplicity. We can handle multi-player general-sum games by considering the product of all the opponents' actions as an augmented action (an idea also recently exploited in~\citep{tian2020provably}). Now we are ready to explain how players interact and learn in the Markov game setup.

\subsection{Vector-valued Markov Games}

\textbf{Model.} Let $[N]:=\{1,2,\ldots,N\}$, and let $\Delta(\X)$ be the set of probability distribution on set $\X$. Then, an episodic two-player zero-sum vector-valued MG can be denoted by the tuple $\mg(\cS, \cA, \cB, \P, \vr, H)$, where 
\begin{list}{–}{\leftmargin=1.5em}
  \setlength\itemsep{-1pt}

\item $H$ is the number of steps in each episode,
\item $\cS$ is the state space, 
\item $\cA$ and $\cB$ are the action spaces of both players,
\item $\P$ is a collection of \emph{unknown} transition kernels $\{\P_h: \cS \times \cA \times \cB \to \Delta(\cS)\}_{h\in [H]}$, and
\item $\vr$ is a collection of \emph{known} $d$-dimensional return functions $\{\vr_h: \cS \times \cA \times \cB \to [0, 1]^d\}_{h\in [H]}$, where $d\ge 2$ is the \emph{dimensionality} of the MG. We assume known $\vr$ only for simplicity; learning $\vr$ poses no real difficulty--see e.g.,~\cite{azar2017minimax, jin2018q}.
\end{list}
Let $|\cdot|$ denote set cardinality. Then, we define the three key cardinalities $S := |\cS|$, $A := |\cA|$, and $B :=  |\cB|$.

\paragraph{Interaction protocol.}
Without loss of generality, in each episode the MG starts at a \emph{fixed} initial state $s_1 \in \cS_1$. At each step $h\in [H]$, the two players observe the state $s_h\in \cS$ and simultaneously take actions $a_h\in \cA$, $b_h\in \cB$. This decision is specified by the players' policies $\mu_h(s_h) \in \Delta(\cA)$ and $\nu_h(s_h) \in \Delta(\cB)$. Then the environment transitions to the next state $s_{h+1} \sim \P_{h}(\cdot\vert s_h, a_h, b_h)$ and outputs the return $\vr_{h}(s_h, a_h, b_h)$. Let $\cF_h^k$ be the filtration generated by all these random variables until the $k$-th episode and $i$-th step.

\paragraph{Value functions.}
Analogous to usual MDPs, for a policy pair $(\mu, \nu)$, step $h\in [H]$, state $s\in \cS$ and actions $a\in \cA, b\in \cB$, we define the State- and Q-value functions as:
\begin{align*}
    &\vV_{h}^{\mu, \nu}(s) := \E_{\mu, \nu}\Bigl[\nlsum_{l=h}^{H} \vr_{l}(s_{l}, a_{l}, b_{l}) \vert s_h = s\Bigr], \\
    &\vQ_{h}^{\mu, \nu}(s, a, b) :=\E_{\mu, \nu}\Bigl[ \nlsum_{l=h}^{H} \vr_{l}(s_{l}, a_{l}, b_{l}) \vert s_h = s, a_h = a, b_h = b \Bigr].
\end{align*}
For compactness of notation, for any $\vV \in [0,H]^{dS}$ and $\vQ \in [0, H]^{dSAB}$ we introduce the operators $\P$ and $\D$ by
\begin{align*}
    \P_h [\vV](s, a, b) := \E_{s'\sim \P_h(\cdot\vert s, a, b)}[\vV(s')], \,\,\,\, \D_{\mu, \nu}[\vQ](s) := \E_{a\sim \mu(\cdot\vert s), b\sim \nu(\cdot\vert s)}[\vQ(s, a, b)].
\end{align*}
With this notation we obtain the Bellman equations:
\begin{align*}
   \vV_{h}^{\mu, \nu}(s) = \D_{\mu_h, \nu_h} [\vQ_{h}^{\mu, \nu}](s), \,\,\,\, \vQ_{h}^{\mu, \nu}(s, a, b) = (r_h + \P_{h}[\vV_{h+1}^{\mu, \nu}])(s, a, b).
\end{align*}
For convenience define $\vV_{H+1}^{\mu, \nu}(s) = 0$ for any $s\in \cS$.

\paragraph{Satisfiability.} 
Let $\W^\star$ denote a desired target set. Henceforth, we assume that $\W^\star$ is a is closed and convex subset of $[0,H]^d$. Let $\hat{\vV}^k$ be the cumulative return received by the agent in the $k$th episode and $\vW^K:=\frac1K\sum_{k=1}^K \hat{\vV}^k$ be the average for the first $K$ episodes. The goal of the agent is to guarantee that $\vW^K \in \W^{\star}$. This goal is achievable under the following satisfiability assumption.
\begin{assumption}[Satisfiability]
   \label{ass:satisfiability}
    Given a vector-valued MG $\mg(\cS, \cA, \cB, \P, \vr, H)$,  we say a closed and convex target set $\W^{\star}$ is satisfiable, if there exists a policy $\mu$ such that for any policy $\nu$, the vector value $\vV_{1}^{\mu, \nu}(s_1) \in \W^{\star}$. 
\end{assumption}

Informally, satisfiability means that the agent can ensure the cumulative return is contained in the target set, regardless of the opponent's action. A weaker notion is if upon knowing the opponent's policy the agent can satisfy the target set. Thus, we call it \emph{Response-satisfiability}.

\begin{assumption}[Response-satisfiability]
   \label{ass:response-satisfiability}
   Given a vector-valued MG $\mg(\cS, \cA, \cB, \P, \vr, H)$, we say a closed and convex target set $\W^{\star}$ is response-satisfiable, if for any policy $\nu$,there exists a policy $\mu$ such that $\vV_{1}^{\mu, \nu}(s_1) \in \W^{\star}$. 
\end{assumption}

Both notions coincide in a scalar-valued zero-sum game, as a result of von Neumann's minimax theorem. However, for vector-valued games, satisfiability is strictly stronger. Indeed, satisfiability fails even in some simple games while response-satisfiability holds. See the discussion in Section 2.1 of \cite{abernethy2011blackwell} for a concrete example.

Without satisfiability, we cannot expect to reach the target set $\W^{\star}$. Luckily, approaching a response-satisfiable set $\W^{\star}$ on average is still possible. To that end, we can reduce the vector-valued MG to a scalar-valued one, as shown below.
 
\subsection{Scalar Reduction and Minimax Theorem} 
We can convert a vector-valued MG to a scalar-valued one by replacing the return vector $\vr$ by the scalar $\vr \cdot \vtheta$, where $\vtheta \in \mathbb{R}^d$ is a fixed vector. Importantly, we will treat $\vtheta$ as a dual variable in our algorithms. For the resulting MG we can define $V_{h}^{\mu, \nu}(\vtheta,s)$ and $Q_{h}^{\mu, \nu}(\vtheta,s,a,b)$ similarly.

We call the two players the ``min-player'' and the ``max-player''\footnote{To accommodate conventions in Approachability, we make the agent the min-player (usually the max-player in MG literature).}. Let $\nu$ be a policy of the max-player. There exists a \emph{best response} $\mu^{\dagger} $ to $\nu$, such that for any step $h\in [H]$ and state $s\in \cS$ we have $V_{h}^{\mu^{\dagger}, \nu}(s) = V_{h}^{\dagger, \nu}(s) := \min_{\mu} V_{h}^{\mu, \nu}(s)$. A symmetric discussion applies to the best response to a min-player's policy. The following minimax equality holds: for any step $h\in [H]$ and state $s\in \cS$, 
\begin{align*}
    \min_{\mu} \max_{\nu} V_{h}^{\mu, \nu}(\vtheta,s) = \max_{\nu} \min_{\mu} V_{h}^{\mu, \nu}(\vtheta,s).
\end{align*}
A policy pair $(\mu^{\star}, \nu^{\star}) $ that achieves the equality is known as a \emph{Nash equilibrium}. We use $V^{\star}_h(\vtheta,s):=V_h^{\mu^{\star},\nu^{\star}}(\vtheta,s)$ to denote the value at the Nash equilibrium, which is unique for the MG and we call the \emph{minimax value} of the MG. 

\paragraph{Approachability.} Scalarizing a vector-valued MG is equivalent to considering a half-space that contains $\W^{\star}$ instead of $\W^{\star}$ itself. If we can reach $\W^{\star}$, then we can reach any half-space that contains $\W^{\star}$. Therefore, satisfiability of half-spaces that contain $\W^{\star}$ is weaker than satisfiability of $\W^{\star}$ itself. We state this condition formally below.

\begin{assumption}[Approachability]
  \label{ass:Approachability}
  Given a vector-valued MG $\mg(\cS, \cA, \cB, \P, \vr, H)$, we say a closed and convex target set $\W^{\star}$ is approachable, if for any vector $\vtheta$,
  $$
  \underset{\vx\in \W^{\star}}{\max}\ \ \vtheta\cdot \vx \ge V_1^{\star}\left(\vtheta, s_1 \right). 
  $$
\end{assumption}

Assumption~\ref{ass:Approachability} is also known as ``half-space satisfiability'' in the literature~\citep{blackwell1956analog}. Indeed, it is equivalent to response-satisfiability (See Lemma 7 in \cite{abernethy2011blackwell}. The proof therein carries over for MGs directly, since it only depends on the geometric property of $\W^{\star}$.). We will only use this approachability condition in the sequel; it results in no loss of generality, and moreover, it is easier to extend to the non-approachable case.

So far we assumed that the target set $\W^{\star}$ is approachable. In practice, this assumption may or may not hold. In both cases, we can still seek to minimize the Euclidean distance $\dist(\vW^K, \W^{\star})$ of the average return to the target set. This is analogous to the agnostic learning setting for supervised learning. Toward this end, the following condition is useful.
\begin{assumption}[$\delta$-Approachability]
  \label{ass:delta-Approachability}
  Given a vector-valued MG $\mg(\cS, \cA, \cB, \P, \vr, H)$, we say a closed and convex target set $\W^{\star}$ is $\delta$-approachable, if for any vector $\vtheta$,
  $$
  \underset{\vx\in \W^{\star}}{\max}\ \ \vtheta\cdot \vx + \delta \ge V_1^{\star}\left(\vtheta, s_1 \right). 
  $$
\end{assumption}
Equivalently, this means the $\delta$-expansion of $\W^{\star}$ is approachable. So, a  larger $\delta$ means $\W^\star$ is harder to approach. 

\section{Multi-objective Meta-algorithm}
\label{sec:basic}
Equipped with the generalized concepts of approachability for vector-valued MGs, we are ready to present our algorithmic framework. To make the exposition modular, we first present \textbf{M}ulti-\textbf{O}bjective \textbf{M}eta-\textbf{A}lgorithm (\BVI{}), our generic learning algorithm that is displayed as Algorithm~\ref{algorithm:Blackwell-VI}. Subsequently, we explain  its key components.

\begin{algorithm}[h]
  \caption{Multi-objective Meta-algorithm (\BVI{})}
  \label{algorithm:Blackwell-VI}
  \begin{algorithmic}[1]
    \State {\bfseries Initialize:} for any $(s, a, b, h,s')$,
    $Q_{h}(s,a, b)\setto \sqrt{d}H$, $N_{h}(s,a, b)\setto 0$, $N_h(s,a,b,s')\setto 0$, $\vW \setto \mathbf{0}$, $\vtheta \setto$ any unit verctor, $\hat{\P} \setto$ any probability distribution.
    \For{Episode $k=1,\dots,K$}
      \State $\pi \setto \plan{}(\vtheta,\vr,N, \hat{\P})$ \label{line:planning}
      \State $\hat{\vV} \setto \mathbf{0}$. \label{line:model_update_start}
      \For{step $h=1,\dots, H$} 
        \State take action $(a_h, \cdot) \sim  \pi_h(\cdot, \cdot| s_h)$. 
        \State Observe opponent's action $b_h\sim \nu_h(s_h)$ and next state $s_{h+1}$. 
        \State $\hat{\vV} \setto \hat{\vV} + \vr_h(s_h, a_h, b_h)$.
        \State $N_{h}(s_h, a_h, b_h)\setto N_{h}(s_h, a_h, b_h) + 1$.
        \State $N_h(s_h, a_h, b_h, s_{h+1}) \setto N_h(s_h, a_h, b_h, s_{h+1}) + 1$
        \State $\hat{\P}_h(\cdot|s_h, a_h, b_h)\setto \frac{N_h(s_h, a_h, b_h, \cdot)}{N_h(s_h, a_h, b_h)}$.
      \EndFor
      \State $\vW \setto ((k-1)\vW+\hat{\vV})/k$. \label{line:model_update_end}
      \State $\vtheta \setto \DU{}(\vW,\W^{\star},\hat{\vV})$ \label{line:dual_update} 
    \EndFor
  \end{algorithmic}
\end{algorithm}

MOMA is partitioned into into three components:
\begin{list}{–}{\leftmargin=1.5em}
   \setlength\itemsep{0em}
\item \textbf{Planning} (Line 3): In each episode, we convert the vector-valued MG into a scalar-valued one by projecting onto the direction specified by the dual variable $\vtheta$ and by computing the policy $\pi$.
\item \textbf{Model Update} (Line 4 to 13): We accumulate the (vector-valued) return in each episode in $\hat{\vV}$, and $\vW$ is the average cumulative return. Then, we update the empirical estimators of the transition kernel.
\item \textbf{Dual Update} (Line 14): Finally, we need to determine which direction we want to project the vector-valued MG onto in the next episode.
\end{list}

Notice that $\pi$ actually defines policies for both players, but we only execute it for the agent.  Let $\mu_h(\cdot| s_h)$ and $\omega _h(\cdot| s_h)$ be the marginal distributions of $\pi_h(\cdot, \cdot| s_h)$. Then action $a_h$ is indeed sampled from the marginal $\mu_h(\cdot| s_h)$, while $b_h$ is sampled from $\nu_h(\cdot| s_h)$, which is not necessarily equal to $\omega _h(\cdot| s_h)$. Using this notation, we can observe that $\hat{\vV}$ is unbiased in the sense that 
$\E[\mathbf\vtheta \cdot  \hat{\vV}] = V_1^{\mu ,\omega}\left(\vtheta, s_1 \right) $. 

The idea behind Algorithm~\ref{algorithm:Blackwell-VI} is simple: In each episode we fix a direction and try to approach the target set $\W^{\star}$. In this way, we can reduce the problem to a scalar-valued MG and benefit from existing work on scalar-valued MGs~\cite{bai2020provable, xie2020learning, bai2020near,liu2020sharp}. The implementation of model updates is described in Algorithm~\ref{algorithm:Blackwell-VI}. The other two sub-procedures vary slightly in different settings as follows:
\begin{list}{–}{\leftmargin=1.5em}
  \setlength\itemsep{0em}
  \item \textbf{\plan{}}: A planning algorithm to determine the policy $\pi$ based on the current estimated transition kernel $\hat{\P}$. For MGs we will use \VIH{} (Algorithm~\ref{alg:VI-Hoeffding}). For MDPs, we can design a finer \VIB{} (Algorithm~\ref{alg:VI-Bernstein}) to achieve a sharper convergence rate. In Line 11 of \VIH{}, we use \NASH{} to denote computing the minimax policy w.r.t. a \emph{matrix} game, which is standard in model-based method for MGs \cite{bai2020provable,xie2020learning, liu2020sharp}. 
  \item \textbf{\DU{}}: A dual update algorithm to update the variable $\vtheta$, which  describes the direction to approach $\W^{\star}$ in the next episode. We propose two different candidates: (\ref{equ:PDU}) and  (\ref{equ:ODU}) in the following two sections. A variant of \ODU{}, \DODU{} is proposed in Section~\ref{sec:satisfiable} to simutaneously optimize a cost function.
\end{list}

\begin{algorithm}[h]
  \caption{VI-Hoeffding ({\VIH{}})}
  \label{alg:VI-Hoeffding}
  \begin{algorithmic}[1]
    \For{step $h=H,H-1,\dots,1$} \label{line:VI_start}
      \For{$(s, a, b)\in\cS\times\cA\times \cB$}
        \State $t \setto N_{h}(s, a, b)$.
        \If{$t>0$}
          \State $r_h(s,a,b)=\vtheta \cdot \vr_h(s,a,b)$;
          \State $\beta \setto c\sqrt{\min\{d,S\}H^2d\iota/t}$. 
          \State $Q_{h}(s, a, b)\setto \max\{(r_h +
          \widehat{\P}_{h} V_{h+1})(s, a, b) - \beta, -\sqrt{d}H\}$. 
        \EndIf
      \EndFor
      \For{$s \in \cS$}
        \State $\pi_h(\cdot, \cdot|s) \setto \NASH (Q_h(s, \cdot, \cdot))$. 
        \State $V_h(s) \leftarrow (\D_{\pi_h}Q_h)(s)$. 
      \EndFor
    \EndFor
  \end{algorithmic}
\end{algorithm}
\section{Projection-based Dual Update}
\label{sec:projection}
We begin with the most intuitive way to choose the dual variable: follow the direction that minimizes the distance of a candidate vector $\vW$ to the target set $\W^{\star}$:  
\begin{equation}
  \label{equ:PDU}
  \vtheta \setto \begin{cases}
    \frac{\vW-\Pi _{\W^{\star}}\left( \vW \right) }{\left\|  \vW-\Pi _{\W^{\star}}\left( \vW \right) \right\|_2 }, \text{if}\ \vW \notin \W^{\star},\\
    \text{any unit vector}, \,\, \text{otherwise}.
  \end{cases}
  \tag{\PDU{}}
\end{equation}
To find this direction, we need to compute the orthogonal projection onto $\W^{\star}$, thus we call it \PDU{}. 

To give theoretical guarantees, we will prove upper bounds on the Euclidean distance from our average cumulative return in the first $K$ episodes $\vW^K$ to the target set $\W^{\star}$. If $\W^{\star}$ is approachable, $\dist(\vW^K,\W^{\star})$ will converge to zero.

\begin{theorem}
  \label{thm:basic-approachable}
  Following \BVI{} with \VIH{} (Algorithm~\ref{alg:VI-Hoeffding}) for \plan{} and \PDU{} for \DU{}, if $\W^\star$ is approachable, with probability $1-p$,
  \begin{align*}
    \dist(\vW^K,\W^{\star}) \le \cO\bigl( \sqrt{\min\{d,S\}dH^4SAB\iota/K}\bigr),
  \end{align*}
  where $\iota =\log(SABKH/p)$.
\end{theorem}

The approachability condition (Assumption~\ref{ass:Approachability}) is standard in the literature \cite{blackwell1956analog}. However in practice, the desired target set $\W^\star$ may rarely also happen to be approachable (since it is chosen to meet the needs of an application, not to meet our demands on approachability). In this case, one may be unable to guarantee $\dist( \vW^K,\W^{\star} )$ converges to zero, but can only minimize the distance. A natural way to model this scenario is to assume $\W^{\star}$ is $\delta$-approachable, whence the following Theorem~\ref{thm:basic-delta-approachable} applies.

\begin{theorem}
  \label{thm:basic-delta-approachable}
  If we use \VIH{} (Algorithm~\ref{alg:VI-Hoeffding}) for \plan{} and (\ref{equ:PDU}) for \DU{} in \BVI{}, and if $W^{*}$ is $\delta$-approachable, then with probability $1-p$,
  \begin{align*}
    \dist \left( \vW^K,W^* \right) \le \delta + \cO\left( \sqrt{\min\{d,S\}dH^4SAB\iota/K}\right)
  \end{align*}
  where $\iota =\log \left( SABKH/p \right) $.
\end{theorem}
\textbf{Remark.} Although we assume $\W^{\star}$ is $\delta$-approachable, the algorithm does not need to know $\delta$. Instead, we just run the same algorithm and the guarantee is adaptive.  
 
\paragraph{Rationale behind the criterion.} When characterizing the performance of our method, we choose to compete with $\delta$, the ``non-approachability gap''. This choice is simple and similar to the notion of regret used in scalar-valued MGs \cite{xie2020learning,tian2020provably}. One may aim to be more ambitious: compete with the best response in hindsight, as in \cite{mannor2014approachability} for the bandit (single-horizon) setting. Unfortunately, such a choice is not computationally feasible for MGs. It is computationally hard even for scalar-valued MGs; see \cite{bai2020near} for an exponential lower bound.

\section{Projection-free Dual Update}
\label{sec:OCO}

The per-iteration computational bottleneck of \PDU{} is to compute the projection onto $\W^{\star}$, which requires solving a quadratic program and can be computationally demanding. However, if we can find $\argmax_{\vx\in \W^{\star}} \vtheta\cdot \vx$ efficiently (e.g., when $\W^{\star}$ is a polytope), then we can develop a computation-friendly dual update based on online convex optimization (OCO) techniques~\citep{abernethy2011blackwell, shimkin2016online}. 

To show the intuition behind \ODU, we proceed via Fenchel duality. Consider a convex, closed, 1-Lipschitz function $f:\left[ 0,H \right] ^d \rightarrow \mathbb{R}$. Its Fenchel conjugate is
$$f^*\left( \vtheta \right) :=\underset{\vx\in X}{\max}\left\{ \vtheta\cdot \vx-f\left( \vx \right) \right\} .
$$

Then $f^*$ is $\sqrt{dH^2}$-Lipschitz by Corollary 13.3.3 in \cite{rockafellar1970convex}. Fenchel duality implies
\begin{equation}
\label{equ:duality}
   f\left( \vx \right) =\underset{\left\| \vtheta \right\| \le 1}{\max}\left\{ \vtheta\cdot \vx-f^*\left(\vtheta \right) \right\} .
\end{equation}

In particular, if $f(\vx) = \dist( \vx,\W^{\star})$, its Fenchel dual is $f^*(\vtheta) = \max_{\vx\in \W^{\star}} \vtheta\cdot \vx$ and its subdifferential is $\partial f^*\left( \vtheta \right) = \argmax_{\vx\in \W^{\star}} \vtheta\cdot \vx$. Therefore, we can use its dual representation to ``linearize'' the distance. That is,
\begin{align*}
  K\dist(\vW^k,\W^{\star}) = \max_{\left\| \vtheta \right\| \le 1}\biggl\{ \vtheta\cdot \sum_{k=1}^K{\mathbf{\hat{V}}^k}- \sum_{k=1}^K \max_{\vx\in \W^{\star}} \vtheta\cdot \vx \biggr\}.  
\end{align*}
Ideally, if we can find the dual variable $\vtheta^{\star}$ that maximizes the right-hand side above, minimizing the distance will be equivalent to minimizing a linear function in $\hat{\vV}^k$, which can be handled as before if we use \VIH{} as the planning algorithm. Although we can not find $\vtheta^{\star}$ directly, we can find a sequence of dual variables $\{\vtheta\}_{k=1}^K$ such that $ \sum_{k=1}^K{\bigl\{ \vtheta^k\cdot \mathbf{\hat{V}}^k - \sum_{k=1}^K \max_{\vx\in \W^{\star}} \vtheta^k\cdot \vx \bigr\}}$ is close to $\max_{\left\| \vtheta \right\| \le 1}\bigl\{ \vtheta\cdot \sum_{k=1}^K{\mathbf{\hat{V}}^k}- \sum_{k=1}^K \max_{\vx\in \W^{\star}} \vtheta\cdot \vx \bigr\}$.

This task is precisely what online convex optimization (OCO) performs. The simplest solution is to use online subgradient method with step size $\eta^k= \sqrt{1/dH^2k}$. We define \ODU{} formally below:
\begin{equation}
  \label{equ:ODU}
  \vtheta^{k+1}:=\Pi_{\B^d}\bigl\{ \vtheta^k+\eta^k \bigl(\mathbf{\hat{V}}^k-\partial f^*\bigl(\vtheta^k \bigr) \bigr) \bigr\},
  \tag{\ODU{}}\vspace{4pt}
\end{equation} 
where $\Pi_{\B^d}$ denotes projection onto the $d$-dimensional unit Euclidean ball and $\partial f^*\bigl(\vtheta^k \bigr)$ is a subgradient vector of $f^*$ at $\vtheta^k$ (not a set).

Similarly, we provide theoretical guarantees for the new dual update rule. The proof is much simpler compared with that of Theorem~\ref{thm:basic-approachable} and Theorem~\ref{thm:basic-delta-approachable}.

\begin{theorem} 
  \label{thm:OCO-delta-approachable}
  Following \BVI{} with \VIH{} (Algorithm~\ref{alg:VI-Hoeffding}) for \plan{} and \ODU{} for \DU{}, if $\W^{\star}$ is $\delta$-approachable, with probability $1-p$,
  \begin{align*}
    \dist \left( \vW^K,\W^{\star} \right) \le \delta+ \cO\left( \sqrt{\min\{d,S\}dH^4SAB\iota/K}\right),   
  \end{align*}
  where $\iota =\log \left( SABKH/p \right) $. 
\end{theorem}

\section{Application to CMDPs: Near Optimal Rate}
\label{sec:MDP_upper}
In this section, we apply our algorithmic framework to MDPs, which can be considered as a special case of MGs where the adversary cannot change the game. The stationary environment enables us to use the Bernstein-type concentration and achieve sharper dependence on the horizon $H$. The corresponding planning algorithm \VIB{} is formalized in Algorithm~\ref{alg:VI-Bernstein}. In Line 6 we use  the empirical variance operator defined by $\widehat{\V}^k_h[V](s,a) :=\Var_{s'\sim \Phat^k_h(\cdot|s,a)}V(s')$ for any function $V \in [-\sqrt{d}H,\sqrt{d}H]^{S}$. Notice that this approach does \emph{not} work for MGs, because we need to estimate the variance of the value function $V^{\mu ,\upsilon}$, a task that is impossible when the adversary's policy $\upsilon$ is unknown.

 \begin{algorithm}[h]
   \caption{\VIB{}}
   \label{alg:VI-Bernstein}
\begin{algorithmic}[1]
   
   \For{step $h=H,H-1,\dots,1$} 
   \For{$(s, a)\in\cS\times\cA$}
   \State $t \setto N_{h}(s, a)$.
   \If{$t>0$}
   \State $r_h(s,a)=\vtheta \cdot \vr_h(s,a)$;
   \State { $\beta \setto c\big(\sqrt{\hat{\V}_h \low{V}_{h+1}(s,a)\min\{d,S\}\iota/t} + \hat{\P}_{h}(\up{V}_{h+1}-\low{V}_{h+1})(s,a)/H + \min\{d,S\}\sqrt{d}H^2\iota/t\big)$. }
   \State { $\low{Q}_{h}(s, a)\setto \max\{(r_h +
   \widehat{\P}_{h} \up{V}_{h+1})(s, a) - \beta, -\sqrt{d}H\}$.} 
   \State {$\up{Q}_{h}(s, a)\setto \min\{(r_h +
   \widehat{\P}_{h} \low{V}_{h+1})(s, a) + \beta, \sqrt{d}H\}$. }
   \EndIf
   \EndFor
   \For{$s \in \cS$}
   \State $\pi_h(s) \setto \argmin (\low{Q}_h(s, \cdot))$. 
   \State $\low{V}_h(s) \leftarrow \low{Q}_h(s, \pi_h(s)),  \up{V}_h(s) \leftarrow \up{Q}_h(s, \pi_h(s))$.
   \EndFor
   \EndFor
\end{algorithmic}
\end{algorithm}

The sharper theoretical guarantee is as follows:

\begin{theorem}
   \label{thm:bernstein}
   If we use \VIB{} (Algorithm~\ref{alg:VI-Bernstein}) for \plan{} and \eqref{equ:PDU} or \eqref{equ:ODU} for \DU{} in \BVI{}, and if $\W^{\star}$ is $\delta$-approachable, then with probability $1-p$,
   $$
   \dist( \vW^K,  \W^*) \le \delta+ \cO\bigl( \sqrt{\min\{d,S\}dH^3SA\iota/K}\bigr),
 $$
 where $\iota =\log \left( SAKH/p \right)$.
\end{theorem}
When $d \le S$ (as is in most cases), our result is minimax optimal up to log-factors in  $S, A, H, K$ according to the lower bound $\Omega \left( \sqrt{H^3SA\iota/K}\right)$ proven in \citep{domingues2020episodic}. The tightness of our result in $d$ remains open. In particular, we can get a naive $\Omega \left( \sqrt{dH^3SA\iota/K}\right)$  lower bound by duplicating the negative  MDP example from  \cite{domingues2020episodic} $d$ times in $d$ dimensions, and the distance naturally scales up by $d$. With such a lower bound, there is still a $\sqrt{d}$ gap open. More details on the difficulty of providing a tigher lower bound are discussed in Section~\ref{sec:conclusion}. 

The upper bound in Theorem~\ref{thm:bernstein} allows us to find a policy that approaches the target set $\W^\star$ efficiently. Next, we generalize the result to the constrained MDP setting where we want to simultaneously minimize a cost function.

\subsection{Optimizing a Cost Function Simultaneously}
\label{sec:satisfiable}

In this section, we show how to extend our algorithm to the constrained MDP setup \citep{efroni2020exploration, ding2020provably, qiu2020upper, brantley2020constrained}, in which one wants to simultaneously minimize a cost function $g : \mathbb{R}^d \to [0,1]$ defined on the return vector space. The goal is two-fold: (i) satisfy constraints defined by the target set; and (ii) minimize the cumulative cost. Note that our setup \emph{subsumes} the canonical cost function in which the cost function is defined on the state-action pair (e.g., \citep{efroni2020exploration}). Particularly, we can add an extra coordinate in the return vector space to denote the cost for each state-action pair, and pick $g$ to solely extract that cost coordinate.  A more detailed comparison against constrained MDP setups from previous works can be found in Appendix~\ref{sec:cmdp-compare}.

For our analysis, we assume that the cost function $g(\cdot)$ is $1$-Lipschitz and convex. Following~\citep{efroni2020exploration,ding2020provably, qiu2020upper, brantley2020constrained}, we also assume $\W^{\star}$ is satisfiable and that we want to compete with a policy $\mu^{\star}$ such that $\vV^{\mu^{\star}}_1(s_1) \in \W^{\star}$. One might hope to bound the regret $\sum_{k=1}^K g(\hat{\vV}^{k})-K g(\vV^{\mu^{\star}}_1(s_1))$. But this goal is hard. Its counterpart is unknown even in the bandit setup  \citet{agrawal2014bandits}. Instead, we aim to upper bound both the regret
$[g(\vW^K)- g(\vV^{\mu^{\star}}_1(s_1))]$ and the constraint violation $\dist(\vW^K, \W^{\star}).$  

\paragraph{Constraint geometry.} Toward achieving our aim, we need to impose some geometric requirements on the constraints that will help us quantify algorithmic complexity in a non-asymptotic manner. Previous works that use a primal-dual approach (e.g., \citep{efroni2020exploration,qiu2020upper,ding2020provably}) assume knowledge of explicit structure of the constraint set, concretely by requiring $\W^{\star} = \{ x \ \|  \forall i, g_i(x) \le 0\}$. Subsequently, they control complexity of the constraint set by assuming Lipschitzness of the $g_i$ and a strong Slater condition, i.e., there is a strictly feasible interior point $x_0$ such that $g_i(x_0) \le -\epsilon$ for a universal constant $\epsilon > 0$. In contrast, we do not impose  explicit structure on $\W^\star$. Instead, we assume that we can solve linear or quadratic optimization over $\W^\star \subset \R^d$. A naive way to cast our setup into the previous form would be use the inequality $g_0(\cdot) := \dist(\cdot, \W^{\star}) \le 0$. But since $g_0$ is a distance function, we cannot satisfy the strict interiority condition needed by the previous setup. Consequently, we need to limit the complexity of our constraint set through a more refined alternative. 

To this end, we propose a geometric condition. In particular, we assume that the target set $\W^{\star}$ intersects with the set of achievable value vectors $\cV = \{{\vV}^\pi_1(s_1) |\ \text{any policy $\pi$}\}$ \emph{nonsingularly}---Figure~\ref{fig:intersection} illustrates this concept. Formally, denote the set of achievable returns within the target set as $\cW = \cV \cap \W^{\star}$ and $\partial \cW = \partial \cV \cap \partial\W^{\star}$ as the intersection of the boundaries of $\W^{\star}$ and the achievable value vector set $\cV$. Then, Assumption~\ref{assump:angle} describes nonsingular intersection.

\begin{assumption}\label{assump:angle}
  If $\partial \cW$ is not empty, then for each vector $\vW \in \partial \cW$, denote the maximum angle $\alpha \in [0, \pi]$ between the support vectors $\vec{a}$ of $\W^{\star}$ at $\vW$ and the support vectors $\vec{b}$ of $\cV$ at $\vW$ as
  \begin{align*}
    \alpha(\vW) := \min\{\angle(\vec{a}, \vec{b})\ |\ \vec{a}, \vec{b}\ \text{are support} \text{vectors of sets}\ \W^{\star}\  \text{and}\ \cV\ \text{at}\ \vW\}.
  \end{align*}
  We assume there exists a constant $\alpha_{\max} \in [\pi/2, \pi)$ such that $\max_{w \in \partial\cW} \alpha(w) < \alpha_{\max}$. With this upper bound on $\alpha$, we denote $\gamma_{\min}= \sin(\pi  - \alpha_{\max}) > 0$.
\end{assumption}

\begin{figure}[h]
	\centering
	\includegraphics[width=0.6\textwidth]{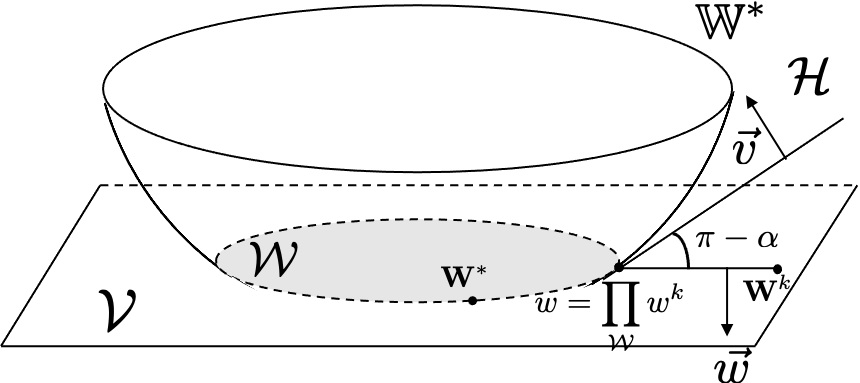}
	\caption{The target set intersects with the achievable return vectors nonsingularly. The angle $\alpha(\prod_\cW \vW^k)$ is upper bounded. }\label{fig:intersection}
\end{figure}

Assumption~\ref{assump:angle} excludes the case where the sets $\cV $ and $\W^{\star}$ intersect tangentially (i.e., share the same supporting hyperplane) resulting in $\alpha = \pi$. The necessity of such a geometric assumption is discussed in Appendix~\ref{sec:intersect-proof}.  At a high level, Assumption~\ref{assump:angle} is a geometric analog of the previously noted strict interiority condition that excludes a singular intersection of the constraint functions $g_i$. Our assumption provides a way to lower-bound the distance to the target set $\W^{\star}$ by the distance to the actual constraint set $\cW=\cV\cap \W^\star$, and thus prevent an algorithm from trading off too much constraint violation in exchange for a lower cost value $g(\vV^k)$. 

To minimize cost and avoid constraint violation simultaneously we need a ``double'' version of dual variable update. This idea is formalized in \DODU{} below:
\begin{align}	
  \bm{\varphi}^{k+1} & = \Pi _{\B^d}\bigl\{ \bm{\varphi }^k+\eta^k \bigl( \mathbf{\hat{V}}^k-\underset{\vx \in W^{\star}}{\mathrm{arg}\max}\ \bm{\varphi }^k \cdot \vx \bigr) \bigr\}, \nonumber\\
  \bm{\phi}^{k+1}   &=  \Pi _{\B^d}\bigl\{ \bm{\phi  }^{k} +\eta^k \bigl( \mathbf{\hat{V}}^k -\partial g^{*}( \bm{\phi}^k) \bigr) \bigr\}, \nonumber\\
  \vtheta^{k+1} &= \rho \bm{\varphi  }^{k+1} + \bm{\phi} \tag{\DODU{}}^{k+1}
\end{align}
where $\Pi_{\B^d}$ denotes projection onto the $d$-dimensional unit Euclidean ball and $\partial g^*\bigl(\bm{\varphi}^k \bigr)$ is a subgradient vector of $g^*$ at $\bm{\varphi}^k$ (not a set).

\begin{table}[t!]
    \centering
    \resizebox{0.95\textwidth}{!}{%
        \begin{tabular}{|c|c|c|c|c|}
      \hline
      \textbf{Algorithms} & \textbf{Regret} & \textbf{Constraint Violation} & \makecell{\textbf{Nonlinear} \\ \textbf{Cost and} \\ \textbf{Constraints}} & \makecell{\textbf{Computionally} \\ \textbf{Efficient}} \\ \hline
      \makecell{OptCMDP-bonus \\ \cite{efroni2020exploration}}& $ \tlO\left(\sqrt{H^4S^2AK}\right)$  & $ \tlO\left(\sqrt{dH^4S^2AK}\right)$ &  &  \checkmark\\ \hline
      \cite{brantley2020constrained}& $ \tlO\left(\sqrt{H^3S^2AK}\right)$ & $ \tlO\left(\sqrt{d^3H^3S^2AK}\right)$ & \checkmark &  \\ \hline
      \makecell{OptPD-CMDP \\ \cite{efroni2020exploration}}& $ \tlO\left(\sqrt{(S^2A + d^2)H^4K}\right)$ & $ \tlO\left(\sqrt{(S^2Ad^2 + d^3)H^4K}\right)$ &  & \checkmark \\ \hline
      \makecell{OPDOP \\  \cite{ding2020provably}} &$ \tlO\left(\sqrt{H^5S^4A^2K}\right)$  & $ \tlO\left(\sqrt{H^5S^4A^2K}\right)$ &  & \checkmark \\ \hline
      \makecell{UCPD \\  \cite{qiu2020upper}} &$ \tlO\left(\sqrt{H^5S^2AK}\right)$  & $ \tlO\left(\sqrt{H^5S^2AK}\right)$ &  & \checkmark \\ \hline
      \textbf{This Paper} \cellcolor{gray!25}& $ \tlO\left( \sqrt{\min\{d,S\}dH^3SAK}\right)$ \cellcolor{gray!25}  & $ \tlO\left( \sqrt{\min\{d,S\}dH^3SAK}\right)$ \cellcolor{gray!25} & \checkmark \cellcolor{gray!25} & \checkmark \cellcolor{gray!25}\\ \hline
    \end{tabular}}
    \caption{Comparison with constrained MDP literature.} \label{tab:comparison}
    \end{table}

Here comes our theoretical guarantee for both constraint violation and regret.

\begin{theorem}
  \label{thm:satisfiable}
  Following \BVI{} with \VIB{} (Algorithm~\ref{alg:VI-Bernstein}) for \plan{} and \DODU{} for \DU{}, if $\W^{\star}$ is approachable and $\mu^{\star}$ is a policy s.t. $\vV^{\mu^{\star}}_1(s_1) \in \W^{\star}$, with probability $1-p$ we can bound the constraint violation and the regret respectively as follows:
  \begin{align*}
  \dist( \vW^K,W^{\star}) & \le \cO\bigl(  \sqrt{\min\{d,S\}dH^3SA\iota/K}\bigr),\\
  g(\vW^K)- g(\vV^{\mu^{\star}}_1(s_1)) & \le \cO\bigl( \rho \sqrt{\min\{d,S\}dH^3SA\iota/K}\bigr),
  \end{align*}
  where $\iota =\log \left( SAKH/p \right) $,  $\rho = 2/\gamma_{\min}$.
\end{theorem}

Known results on constrained MDP problems do not share a common setup and hence make a precise comparison tricky. In short, our result aims to provide a \emph{computationally efficient} algorithm for non-linear constraints (target set) and a convex cost function (see Table~\ref{tab:comparison}).  Please see Appendix ~\ref{sec:comparsion} for a more detailed discussion of the subtleties among different constrained MDP setups, and some minor modifications needed to unify the exposition. With the existing results, our result is significant in the following aspects:
\begin{itemize}
  \setlength{\itemsep}{0em}
  \item First, our algorithm is the most general in terms of being able to handle non-linearity in the cost and constraints. The constrained MDP setting we study in Section~\ref{sec:satisfiable} is a direct generalization of~\citep{agrawal2014bandits}, and is closest to~\citep{brantley2020constrained}. While our constraint assumption is equivalent to the one in~\citep{brantley2020constrained}, our cost functions are more general. The domain of \citeauthor{brantley2020constrained}'s cost function is scalars, while that of ours is vectors.
  \item Furthermore, our proposed algorithm is \emph{computationally efficient} because we do not require solving a large-scale convex optimization sub-problem with the number of variables and constraints scaling as $\cO(SAH)$ per iteration (see Table~\ref{tab:comparison}). Indeed, our algorithms only comprise planning and model update procedures with a total of $\cO(S^2AH)$ basic algebraic updates in each episode, along with a dual space optimization procedure whose computational complexity is free of $S$, $A$ and $H$. 
  \item  Our bounds on regret and constraint violation are also the sharpest with respect to their dependence on the parameters $S$, $A$, and $K$.
\end{itemize}

\section{Conclusion and Future Work}\label{sec:conclusion}
In this paper, we formulate online learning in vector-valued Markov games via the lens of approaching a fixed convex target set within which the vector-valued objective should lie. We provide efficient model-based algorithms as instances of a generic meta-algorithm. Two key ideas contribute to our algorithmic design: (i) reduction of the vector-valued Markov game to a scalar-valued one, where the scalarization is iteratively updated; and (ii) exploration of the environment strategically. For vector-valued MDPs, our algorithms, after some modifications, achieve a tight rate in approaching the target set (in terms of $S, A, H, K$), while simultaneously minimizing a convex cost function. Moreover, when the given target set is non-approachable, our algorithms automatically adapt to the degree of non-approachability.

Several problems are left open. Currently, there is still a $\sqrt{d}$ gap ($d$ is the dimensionality of the vector-valued cost) between our upper bound and the lower bound. How to close this gap to achieve the minimax rate remains unknown. The challenge in providing a tighter lower bound is that estimating a discrete distribution under the $L^2$ distance does not get harder as the dimensionality increases. Since we use the Euclidean distance to measure the performance of our algorithms, we cannot get stronger dependence on $d$. Lower bounds such as the one in \citep{jin2020reward} use a multiple hypothesis testing approach successfully because they work with an $L^1$ loss, whereas we study the standard Euclidean loss.  A second question is that our result in Section~\ref{sec:satisfiable} has somewhat worse dependence on $d$ and $\rho$ compared to previous results. We leave improving the dimension dependency as a future direction.  

Another future direction that is worth pursuing is that of redefining the notion of regret and error. Our work measures approachability error using the Euclidean distance. In practice, this choice may not be the only useful measure. Can we develop provably efficient algorithms under other geometries and measures of approachability? Answering this question might help exploit the geometry of the target set better, and potentially lead to tighter complexity analyses.

\bibliographystyle{plainnat}
\bibliography{ref}

\clearpage
\appendix

\section{Proofs for Section~\ref{sec:basic} }
\label{sec:basic_proof}

In this section, we give detailed proofs needed in Section~\ref{sec:basic}. 

Beginning with a recapitulation of the notaions, We denote $V^k$, $Q^k$, $\pi^k$, $\mu^k$, $\nu^k$ and $\theta^k$ for values, policies and dual vectors at the \emph{beginning} of the $k$-th episode, In particular,
$Q^k_h(s, a, b)  = \theta^k \cdot \  Q^{\mu_k, \nu_k}_h(s, a, b).$

$\hat{\vV}^k$ is the cumulative reward in the $k$-th episode . In particular, $N_h^k(s,a,b)$ is the number we have visited the state-action tuple $(s,a,b)$ at the $h$-th step before the $k$-th episode. $N_h^k(s,a,b,s')$ is defined by the same token. Using this notation, we can further define the empirical transition and exploration bonus by
 $
 \widehat{\P}^k_h(s'|s, a, b):= N^k_h(s, a, b, s')/N^k_h(s, a, b)
 $ and $
	\beta_h^k(s,a,b) := C\sqrt{\min\{d,S\}d\iota H^2/N_h^k(s,a,b)} 
$.

We first give a uniform convergence guarantee, which will also be used later. The first simple lemma is from \cite{liu2020sharp}.

\begin{lemma} \label{lem:lipschitz-matrix}
   Let $\X,\Y,\mat{Z}\in\mathbb{R}^{A\times B}$ and $\Delta_d$ be the $d$-dimensional simplex.
   Suppose $\left|\X-\Y \right|\le \mathbf{Z}$, where the inequality is entry-wise. Then
      \begin{equation}
         \left|\max_{\mu\in\triangle_A}\min_{\nu\in\triangle_B}
         \mu\trans \X \nu 
         -\max_{\mu\in\triangle_A}\min_{\nu\in\triangle_B}
         \mu\trans \Y \nu \right|
         \le  \max_{i,j} \mathbf{Z}_{ij}.
      \end{equation}
   \end{lemma}

We also need to following lemma to characterize the dependence of $V^{\star}(\vtheta^k,\cdot)$ on $\vtheta^k$ to apply the covering argument.

\begin{lemma} [Lipschitz property of $V^{\star}$]
   \label{lem:lipschitz-markov}
   For any $s \in \cS$,
   $$
   \left| V^{\star}_h(\vtheta,s)-V^{\star}_h(\vtheta',s) \right|\le \sqrt{d} (H-h+1)\left\| \vtheta-\vtheta' \right\| _2
$$
\end{lemma}
\begin{proof}
   By Cauchy-Schwarz, $| V^{\star}_h(\vtheta,s)| \le \sqrt{d} (H-h+1)$. The rest of proof follows by induction via Bellman equation and Lemma~\ref{lem:lipschitz-matrix}.
\end{proof}

Equipped with this Lipschitz property, we are ready to prove a uniform concentration result. Notice $\B^d$ is the $d$-dimensional unit Euclidean ball centered at $0$.

\begin{lemma}[Uniform Concentration of $V^{\star}(\vtheta,\cdot)$]
   \label{lem:uniform_V_star}
   Consider value function class
   \begin{equation*}
     \cV_{h+1} =
     \set{V:\cS\to\R~\mid~V(\cdot)=V_{h+1}^{\star}\left( \vtheta,\cdot \right) ~\textrm{for
         all}~ \vtheta\in \B^m}.
   \end{equation*}
   There exists an absolute constant $c$, with probability at least $1-p$, for all $(s, a, b,
   k, h)$ and all $V\in \cV_{h+1}$ we have:
   \begin{equation*}
     \abs{[(\hat{\P}^k_h - \P_h)V](s, a, b)} \le
     c \sqrt{\frac{\min\{d,S\} dH^2\iota }{N_h^k(s,a)}} . 
   \end{equation*}
   where $\iota = \log(mSABKH /p)$ is a logarithmic factor.
 \end{lemma}
 \begin{proof}
Let $\mathcal{D}_\eps$ be an $\eps$-covering of $\B^d$ in the $\cL^2$ norm, i.e., for any $\vtheta \in \B^d$ there exists
$\hat{\vtheta}\in \mathcal{D}_\eps$ such that 
$\left\| \vtheta-\hat{\vtheta} \right\| _2 \le \eps$. For each $\hat{\vtheta}\in \mathcal{D}_\eps$, we can define the corresponding value function $V_{h+1}^{\star}\left( \hat{\vtheta},\cdot \right)$. In this way, by Lemma~\ref{lem:lipschitz-markov}, we can generate a set $\cV_\eps$ which is an $H\eps$-covering of $\cV_{h+1}$ in infinity norm, i.e., for any $V \in \cV_{h+1}$ there exists
$\hat{V}\in \cV_\eps$ such that for any $s \in \cS$,
$|V(s)-\hat{V}(s)| \le H\eps$ .

Since $|\mathcal{D}_\eps|\le (1/\eps)^d$, we also have $|\cV_\eps|\le (1/\eps)^d$. Since
$| V^{\star}_h(\vtheta,s)| \le \sqrt{d} (H-h+1)$, by Hoeffding inequality and taking union bound, with probability at least $1-p$,   \begin{align*}
     \abs{\sup_{\hat{V}\in\cV_\eps} [(\hat{\P}^k_h - \P_h)V](s, a, b)} \le \cO \left(  \sqrt{\frac{d^2H^2\iota' }{N_h^k(s,a)}} \right).
\end{align*}
where $\iota' = \iota +\log 1/\eps$.

At the same time, for any $V\in\cV_{h+1}$,
there exists $\hat{V}\in\cV_\eps$ such that $\sup_s |V(s) -
\hat{V}(s)|\le \sqrt{d}H\eps$. Therefore, 
$$
\abs{[(\hat{\P}^k_h - \P_h)V](s, a, b)} \le \abs{[(\hat{\P}^k_h - \P_h)\hat{V}](s, a, b)} +\sqrt{d}H \eps.
$$

Taking $\eps = d\iota/N_h^k(s,a)$ proves 
\begin{equation*}
   \abs{[(\hat{\P}^k_h - \P_h)V](s, a, b)} \le
   c \sqrt{\frac{d^2H^2\iota }{N_h^k(s,a)}} . 
 \end{equation*}
 
 Similarly we also have (for example see Lemma 12 in \cite{bai2020provable})
 \begin{equation*}
   \abs{[(\hat{\P}^k_h - \P_h)V](s, a, b)} \le
   c \sqrt{\frac{ dSH^2\iota }{N_h^k(s,a)}}, 
 \end{equation*} 
 \end{proof}
which completes the proof.

Using the concentration result, we can prove the "lower confidence bounds" are indeed lower bounds with high probability. To do this, we need to introduce a little more notation.

Similar to $V^{\star}_h$, we can also define $Q^{\star}$. By Bellman equation we have
 $$Q^{\star}_h(\vtheta,s,a,b)=[\vtheta \cdot \vr_h+ \P_hV^{\star}_{h+1}(\vtheta,\cdot)](s,a,b).$$ 

\begin{lemma}[Upper confidence bound]
   \label{lem:ULCB}
	With probability $1-p$, for all $h,s,a,b$ and $k\in[K]$, we have
   \begin{equation}
      Q^{k}_h(s,a,b) \le Q^{\star}_h(\vtheta^k,s,a,b),  \,\,\,\, 
		V^{k}_h(s) \le V^{\star}_h(\vtheta^k,s).
	\end{equation}
\end{lemma}
\begin{proof}
   Again, the proof is by backward induction. Suppose the bounds hold for the Q-values in the $(h+1)$-th step, we now establish the bounds for the values in the $(h+1)$-th step and Q-values in the $h$-step. Consider a fixed state $s$,
	  
    \begin{equation}
        \begin{aligned}
            V^{k}_{h+1}(s)&=
            \D_{\pi^k_h} Q^{k}_{h+1}(s)\\
            &= \min_{\upsilon} \D_{\mu^k_{h+1} \times \upsilon}   Q^{k}_{h+1}(s)\\
            &\le \min_{\upsilon} \D_{\mu^k_{h+1} \times \upsilon}   Q^{\star}_h(\vtheta^k,\cdot,\cdot,\cdot)(s)\\
             &\le V^{\star}_h(\vtheta^k,s).
        \end{aligned}
        \end{equation}

	Now consider a fixed triple $(s,a,b)$ at $h$-th step. We have    
    \begin{equation}\label{eq:Q-decomposition}
        \begin{aligned}
			Q^{k}_h(s,a,b) - Q^{\star}_h(\vtheta^k,s,a,b) 
			= &(\widehat{\P}_h^kV^{k}_{h+1} 
            - \P_h V^{\star}_{h+1}(\vtheta^k,\cdot)  - \beta_h^k)(s,a,b)\\
            \overset{\left( i \right)}{\le} & [(\widehat{\P}_h^k- \P_h) V^{\star}_{h+1}(\vtheta^k,\cdot)] (s,a,b) - \beta_h^k(s,a)\\
            \overset{\left( ii \right)}{\le}& 0.
        \end{aligned}
        \end{equation}
   where $(i)$ is by induction hypothesis and $(ii)$ is by Lemma~\ref{lem:uniform_V_star} and the definition of $\beta$.
\end{proof}

A handy decomposition will help us simplify the target we want to bound in Theorem~\ref{thm:basic-approachable} and Theorem~\ref{thm:basic-delta-approachable}. To simplify the notation, when there are no confusion, we use the shorthand $V^{\mu ^k,\nu ^k}$ and $Q^{\mu ^k,\nu ^k}$ for $\vtheta^k \cdot \vV^{\mu ^k,\nu ^k}$ and $\vtheta^k \cdot \vQ^{\mu ^k,\nu ^k}$.

\begin{lemma}[Regret decomposition]
   \label{lem:regret-decomposition}
   The "regret" $[V_{1}^{\mu ^k,\nu ^k}-V_{1}^{k}](s_{1}^{k})$ can be decompsed into
   $$
[V_{1}^{\mu ^k,\nu ^k}-V_{1}^{k}](s_{1}^{k})\le \sum_{h=1}^H{\left( \beta _{h}^{k}+\xi _{h}^{k}+\zeta _{h}^{k} \right)}
$$
where 
\begin{align*}
   \xi _{h}^{k}:=&\left(\D_{\mu ^k\times \nu ^k}Q_{h}^{\mu ^k,\nu ^k}- \D_{\mu ^k\times \nu ^k}Q_{h}^{k} \right) (s_{h}^{k})-\left( Q_{h}^{\mu ^k,\nu ^k}-Q_{h}^{k} \right) \left( s_{h}^{k},a_{h}^{k},b_{h}^{k} \right) \in \left[ -4\sqrt{d}H,4\sqrt{d}H \right], 
   \\
   \zeta _{h}^{k}:=&\P_h\left( V_{h+1}^{\mu ^k,\nu ^k}-V_{h+1}^{k} \right) \left( s_{h}^{k},a_{h}^{k},b_{h}^{k} \right) -\left( V_{h+1}^{\mu ^k,\nu ^k}-V_{h+1}^{k} \right) \left( s_{h}^{k} \right) \in \left[ -4\sqrt{d}H,4\sqrt{d}H \right]  
\end{align*}
are martingale difference sequences adapted to $\cF_h^k$.
\end{lemma}
\begin{proof}
   We have
\begin{align*}
   [V_{h}^{\mu ^k,\nu ^k}-V_{h}^{k}](s_{h}^{k})
   =&\left( \D_{\mu ^k\times \nu ^k}Q_{h}^{\mu ^k,\nu ^k}-\D_{\pi ^k}Q_{h}^{k} \right) (s_{h}^{k})
   \\
   \overset{\left( i \right)}{\le}&\left( \D_{\mu ^k\times \nu ^k}Q_{h}^{\mu ^k,\nu ^k}-\D_{\mu ^k\times \nu ^k}Q_{h}^{k} \right) (s_{h}^{k})
   \\
   =&\left( Q_{h}^{\mu ^k,\nu ^k}-Q_{h}^{k} \right) \left( s_{h}^{k},a_{h}^{k},b_{h}^{k} \right) +\xi _{h}^{k}
   \\
   =&\P_h\left( V_{h+1}^{\mu ^k,\nu ^k}-V_{h+1}^{k} \right) \left( s_{h}^{k},a_{h}^{k},b_{h}^{k} \right) +\beta _{h}^{k}+\xi _{h}^{k}
   \\
   =&\left( V_{h+1}^{\mu ^k,\nu ^k} -V_{h+1}^{k}\right) \left( s_{h+1}^{k} \right) +\beta _{h}^{k}+\xi _{h}^{k}+\zeta _{h}^{k}.
\end{align*}
where $(i)$ is by the definition of Nash equilibirum.

Repete the recursion we have 
$$
[V_{1}^{\mu ^k,\nu ^k}-V_{1}^{k}](s_{1}^{k})\le \sum_{h=1}^H{\left( \beta _{h}^{k}+\xi _{h}^{k}+\zeta _{h}^{k} \right)}.
$$
\end{proof}

The sum of the exploration bonus can be bounded easily.

\begin{lemma}[Sum of bonus]
   \label{lem:sum-of-bonus}
   $$
   \sum_{k=1}^K{\sum_{h=1}^H{\beta _{h}^{k}}}\le O\left( \sqrt{\min\{d,S\}dH^4SABK\iota} \right) 
$$
\end{lemma}
\begin{proof}
   By definition of $\beta$ and pigeonhole principle,
   \begin{align*}
      \sum_{k=1}^K{\sum_{h=1}^H{\beta _{h}^{k}}}\le& \sum_{k=1}^K{\sum_{h=1}^H{O\left( \sqrt{\frac{\min\{d,S\}dH^2\iota}{N_{h}^{k}\left( s_{h}^{k},a_{h}^{k},b_{h}^{k} \right)}} \right)}}
\\
\le& O\left( \sum_{h=1}^H{\sum_{s,a,b}^{}{\sum_{t=1}^{N_{h}^{K}\left( s,a,b \right)}{\sqrt{\frac{\min\{d,S\}dH^2\iota}{t}}}}} \right) 
\\
\le& O\left( \sum_{h=1}^H{\sum_{s,a,b}^{}{\sqrt{\min\{d,S\}dH^2\iota N_{h}^{K}\left( s,a,b \right)}}} \right) 
\\
\le& O\left( \sqrt{\min\{d,S\}dH^4SABK\iota} \right) .
   \end{align*}
\end{proof}

Now we are ready to prove Theorem~\ref{thm:basic-approachable} and Theorem~\ref{thm:basic-delta-approachable}.

\begin{proof}[Proof of Theorem~\ref{thm:basic-approachable}]
   The squared distance can be demcoposed by 
   \begin{align*}
  \dist\left( \vW^k,W^{\star} \right) ^2=&\lVert \vW^k-\Pi _{W^*}\left( \vW^k \right) \rVert _{2}^{2}
\\
\overset{\left( i \right)}{\le} & \lVert \vW^k-\Pi _{W^{\star}}\left( \vW^{k-1} \right) \rVert _{2}^{2}
\\
=&\lVert \frac{k-1}{k}\vW^{k-1}+\frac{1}{k}\mathbf{\hat{V}}^k-\Pi _{W^{\star}}\left( \vW^{k-1} \right) \rVert _{2}^{2}
\\
=&\left( \frac{k-1}{k} \right) ^2\dist\left( \vW^{k-1},W^{\star} \right) ^2+\frac{1}{k^2}\underset{\left( A \right)}{\underbrace{\left\| \mathbf{\hat{V}}^k-\Pi _{W^{\star}}\left( \vW^{k-1} \right) \right\| _{2}^{2}}}\\
&+\frac{2\left( k-1 \right)}{k^2}\underset{\left( B \right)}{\underbrace{\left( \vW^{k-1}-\Pi _{W^{\star}}\left( \vW^{k-1} \right) \right) \cdot \left( \mathbf{\hat{V}}^k-\Pi _{W^{\star}}\left( \vW^{k-1} \right) \right) }} 
\end{align*}
where $(i)$ is by the definition of (Euclidean) projection.

By boundedness of distance, $(A)=\lVert \mathbf{\hat{V}}^k-\Pi _{W^\star}\left( \vW^{k-1} \right) \rVert _{2}^{2}\le dH^2$.

To bound $(B)$, we notice if $\vW^{k-1} \in W^{\star}$, $\dist\left( \vW^{k-1},W^{\star} \right)$ and 

$$
(B)= 0 = \dist\left( \vW^{k-1},W^{\star} \right) \vtheta^k\cdot \left(\mathbf{\hat{V}}^k-\Pi _{W^{\star}}\left( \vW^{k-1} \right)  \right) 
$$
for any $\vtheta^k$. Otherwise, 

\begin{align*}
(B)= &\left( \vW^{k-1}-\Pi _{W^{\star}}\left( \vW^{k-1} \right) \right) \cdot \left( \mathbf{\hat{V}}^k-\Pi _{W^{\star}}\left( \vW^{k-1} \right) \right) 
\\
=&\dist\left( \vW^{k-1},W^{\star} \right) \vtheta^k\cdot \left( \mathbf{\hat{V}}^k-\Pi _{W^{\star}}\left( \vW^{k-1} \right)  \right) 
\\
\overset{\left( i \right)}{\le}&\dist\left( \vW^{k-1},W^{\star} \right) \left( \vtheta^k\cdot \mathbf{\hat{V}}^k-V_{1}^{\star}\left( \vtheta^k,s_{1}^{k} \right) \right) 
\\
\overset{\left( ii \right)}{\le}&\dist\left( \vW^{k-1},W^{\star} \right) \left(\vtheta^k\cdot \mathbf{\hat{V}}^k- V_{1}^{k}\left( s_{1}^{k} \right) \right) 
\\
=&\dist\left( \vW^{k-1},W^{\star} \right) \left[ \left( V_{1}^{\mu ^k,\nu ^k}-V_{1}^{k} \right) \left( s_{1}^{k} \right) +\left( \vtheta^k\cdot \mathbf{\hat{V}}^k -V_{1}^{\mu ^k,\nu ^k}\left( s_{1}^{k} \right)\right)  \right] 
\\
\overset{\left( iii \right)}{\le}&\dist\left( \vW^{k-1},W^{\star} \right) \left[ \sum_{h=1}^H{\left( \beta _{h}^{k}+\xi _{h}^{k}+\zeta _{h}^{k} \right)}+\left( \vtheta^k\cdot \mathbf{\hat{V}}^k- V_{1}^{\mu ^k,\nu ^k}\left( s_{1}^{k} \right)\right)  \right]   
\end{align*}
where $(i)$ is by Approachability (Assumption~\ref{ass:Approachability}), $(ii)$ is by optimism (Lemma~\ref{lem:ULCB}) and $(iii)$ is by regret decomposition (Lemma~\ref{lem:regret-decomposition}). 

Putting everything together and repete the recursion we have 

$$
K\dist\left( \vW^K,W^* \right) ^2\le dH^2+2\sum_{k=1}^K{\frac{k-1}{K}}\dist\left( \vW^{k-1},W^{\star} \right) \left[ \sum_{h=1}^H{\left( \beta _{h}^{k}+\xi _{h}^{k}+\zeta _{h}^{k} \right)}+\left(\vtheta^k\cdot \mathbf{\hat{V}}^k- V_{1}^{\mu ^k,\nu ^k}\left( s_{1}^{k} \right)  \right) \right] 
$$

Now we can begin to prove the theorem by induction. Suppose $$
\dist \left( \vW^k,W^{\star} \right) \le c_0 \sqrt{\min\{d,S\}dH^4SAB\iota/k}.
$$ for $\forall k \le K-1$, let's prove the claim holds for $k=K$. We first consider the optimistic bonus.

\begin{align*}
   \sum_{k=1}^K{\sum_{h=1}^H{\frac{k-1}{K}\dist\left( \vW^{k-1},W^{\star} \right) \beta _{h}^{k}}}\le& c_0\sqrt{\min\{d,S\}dH^4SAB\iota}\sum_{k=1}^K{\sum_{h=1}^H{\frac{\sqrt{k}}{K}\beta _{h}^{k}}}
\\
\le& c_0\sqrt{\min\{d,S\}dH^4SAB\iota /K}\sum_{k=1}^K{\sum_{h=1}^H{\beta _{h}^{k}}}
\\
\overset{\left( i \right)}{\le}&c_0c_1\min\{d,S\}dH^4SAB\iota 
\end{align*}
where $(i)$ is by Lemma~\ref{lem:sum-of-bonus} and $c_1$ is the constant coefficient there.

The remaining terms are martingale difference sequence, so we only need to bound the variance. 

\begin{align*}
   \sum_{k=1}^K{\sum_{h=1}^H{\frac{k-1}{K}\dist\left( \vW^{k-1},W^{\star} \right) \xi _{h}^{k}}}\overset{\left( i \right)}{\le}&c_2\sqrt{\sum_{k=1}^K{\left( \frac{k-1}{K} \right) ^2}\dist\left( \vW^{k-1},W^{\star} \right) ^2dH^3\iota}
\\
\le& c_2c_0\sqrt{\min\{d,S\}d^2H^4SAB\iota ^2}\sqrt{\sum_{k=1}^K{\frac{kH^3}{K^2}}}
\\
\le& c_0c_1\sqrt{\min\{d,S\}d^2H^7SAB\iota ^2}
\end{align*}
where $(i)$ is by Azuma-Hoeffding. Similarly, $\sum_{k=1}^K{\sum_{h=1}^H{\frac{k-1}{K}\dist\left( \vW^{k-1},W^{\star} \right) \zeta _{h}^{k}}}\le c_2c_0\sqrt{\min\{d,S\}d^2H^7SAB\iota ^2}$.

The last term can be handled similarly but we need to be more carefully because different coordinates of $\mathbf{\hat{V}}^k$ are correlated.  

\begin{align*}
   \sum_{k=1}^K{\frac{k-1}{K}\dist\left( \vW^{k-1},W^{\star} \right) \left(\vtheta^k\cdot \mathbf{\hat{V}}^k- V_{1}^{\mu ^k,\nu ^k}\left( s_{1}^{k} \right)  \right)} \le& c_2\sum_{j=1}^m{\sqrt{\sum_{k=1}^K{\left( \frac{k-1}{K} \right) ^2}\dist\left( \vW^{k-1},W^{\star} \right) ^2\left( \vtheta^k_j \right) ^2H^2\iota}}
\\
\le& c_2c_0\sqrt{\min\{d,S\}dH^4SAB\iota ^2}\sum_{j=1}^m{\sqrt{\sum_{k=1}^K{\left( \vtheta^k_j \right) ^2\frac{kH^2}{K^2}}}}
\\
\overset{\left( i \right)}{\le}& c_2c_0\sqrt{\min\{d,S\}dH^4SAB\iota ^2}\sqrt{m\sum_{j=1}^m{\sum_{k=1}^K{\left( \vtheta^k_j \right) ^2\frac{kH^2}{K^2}}}}
\\
=&c_2c_0\sqrt{\min\{d,S\}d^2H^4SAB\iota ^2}\sqrt{\sum_{k=1}^K{\frac{kH^2}{K^2}}}
\\
\le& c_2c_0\sqrt{\min\{d,S\}d^2H^6SAB\iota ^2}
\end{align*}
where $(i)$ is by Cauchy-Schwarz.

After taking a union bound w.r.t. $[K]$, to prove the claim for $k=K$, we only need to guarantee 
$$
dH^2+8c_0\max \left\{ c_1,c_2 \right\} \min\{d,S\}dH^4SAB\iota \le c_0^2\min\{d,S\}dH^4SAB\iota 
$$
which is satisfied as long as $c_0\ge \max \left\{ 16\max \left\{ c_1,c_2 \right\} ,\sqrt{\frac{2}{SABH^2\iota}} \right\}$.
\end{proof}

We can prve Theorem~\ref{thm:basic-delta-approachable} similarly.

\begin{proof}[Proof of Theorem~\ref{thm:basic-delta-approachable}]

As in the proof of Theorem~\ref{thm:basic-approachable} we have 
$$
K\dist\left( \vW^K,W^* \right) ^2\le dH^2+2\sum_{k=1}^K{\frac{k-1}{K}}\dist\left( \vW^{k-1},W^{\star} \right) \left[\delta+ \sum_{h=1}^H{\left( \beta _{h}^{k}+\xi _{h}^{k}+\zeta _{h}^{k} \right)}+\left(\vtheta^k\cdot \mathbf{\hat{V}}^k- V_{1}^{\mu ^k,\nu ^k}\left( s_{1}^{k} \right)  \right) \right] 
$$

Again we prove the theorem by induction. Suppose $$
\dist \left( \vW^k,W^{\star} \right) \le \delta+ c_0 \sqrt{\min\{d,S\}dH^4SAB\iota/k}
$$ for $\forall k \le K-1$, let's prove the claim holds for $k=K$. Now we have a new term to bound, which is
$$
2\delta \sum_{k=1}^K{\frac{k-1}{K}}\left[ \delta +c_0\sqrt{\min \{d,S\}dH^4SAB\iota /k}+\sum_{h=1}^H{\left( \beta _{h}^{k}+\xi _{h}^{k}+\zeta _{h}^{k} \right)}+\left( \boldsymbol{\theta }^k\cdot \mathbf{\hat{V}}^k-V_{1}^{\mu ^k,\nu ^k}\left( s_{1}^{k} \right) \right) \right] =\left( A \right) +\left( B \right) +\left( C \right) 
$$
where 
$$
\left( A \right) =2\delta ^2\sum_{k=1}^K{\frac{k-1}{K}}\le \left( K-1 \right) \delta ^2,
$$
$$
\left( B \right) =2\delta \sum_{k=1}^K{\frac{k-1}{K}}c_0\sqrt{\min \{d,S\}dH^4SAB\iota /k}\le \frac{4}{3}c_0\delta \sqrt{\min \{d,S\}dH^4SAB\iota K},
$$
and by Lemma~\ref{lem:sum-of-bonus} and Azuma-Hoeffding inequality
$$
\left( C \right) =2\delta \sum_{k=1}^K{\frac{k-1}{K}}\left[ \sum_{h=1}^H{\left( \beta _{h}^{k}+\xi _{h}^{k}+\zeta _{h}^{k} \right)}+\left( \boldsymbol{\theta }^k\cdot \mathbf{\hat{V}}^k-V_{1}^{\mu ^k,\nu ^k}\left( s_{1}^{k} \right) \right) \right] \le c_1c_0\delta \sqrt{\min \{d,S\}dH^4SAB\iota K}.
$$

To prove the induction hypothesis, we only need to guarantee
\begin{align*}
   &H^2+8c_0\max \left\{ c_1,c_2 \right\} \min \{d,S\}dH^4SAB\iota +\left( K-1 \right) \delta ^2+c_1\delta \sqrt{\min \{d,S\}dH^4SAB\iota K}
\\
\le& K\left[ \delta +c_0\sqrt{\min \{d,S\}dH^4SAB\iota /K} \right] ^2.
\end{align*}

Comparing the coefficients, we can see this is satisfied by setting  $c_0\ge \max \left\{ c_1,\sqrt{\frac{2}{SABH^2\iota}} \right\}$ .

\end{proof}

\section{Proof for Section~\ref{sec:OCO} }

\begin{proof}[Proof of Theorem~\ref{thm:OCO-delta-approachable}]

The guarantee of online sub-gradient descent yields with high probability    
    
\begin{align}
\label{equ:sgd-regret}
    \underset{\left\| \vtheta \right\| \le 1}{\max}\left\{ \vtheta\cdot \sum_{k=1}^K{\mathbf{\hat{V}}^k}- \sum_{k=1}^K\underset{\vx\in \W^{\star}}{\max}\vtheta\cdot \vx \right\} \le \sum_{k=1}^K{\left\{ \vtheta^k\cdot \mathbf{\hat{V}}^k-\sum_{k=1}^K\underset{\vx\in \W^{\star}}{\max}\vtheta^k\cdot \vx \right\}}+\cO\left( \sqrt{dH^2K} \right) .
\end{align}

Since $f\left( \vx \right) =\dist\left( \vx,\W^{\star} \right) $ is a closed, 1-Lipschitz convex function, the dual representation implies 

\begin{align*}
&K\dist\left( \vW^k,W^{\star} \right)
\\  
=&\underset{\left\| \vtheta \right\| \le 1}{\max}\left\{ \vtheta\cdot \sum_{k=1}^K{\mathbf{\hat{V}}^k}- \sum_{k=1}^K\underset{\vx\in W^{\star}}{\max}\vtheta\cdot \vx \right\} 
\\
\le & \sum_{k=1}^K{\left\{ \vtheta^k\cdot \mathbf{\hat{V}}^k-\sum_{k=1}^K\underset{\vx\in W^{\star}}{\max}\vtheta^k\cdot \vx \right\}}+\cO\left( \sqrt{dH^2K} \right) 
\\
\overset{\left( i \right)}{\le}&\sum_{k=1}^K{\left\{ \vtheta^k\cdot \mathbf{\hat{V}}^k-V_{1}^{\star}\left( \vtheta^k,s_1 \right) +\delta \right\}}+\cO\left( \sqrt{dH^2K} \right) 
\\
\overset{\left( ii \right)}{\le}&K\delta +\sum_{k=1}^K{\vtheta^k\cdot \left( \mathbf{\hat{V}}^k-\vV_1^{\pi ^k,\mu ^k}\left( s_1 \right) \right)}+\sum_{k=1}^K{\left\{ \vtheta^k\cdot \vV_1^{\pi ^k,\mu ^k}\left( s_1 \right) -V_{1}^{k}\left( s_1 \right) \right\}}+\cO\left( \sqrt{dH^2K} \right) 
\\
\overset{\left( iii \right)}{\le}&K\delta +\sum_{k=1}^K{\vtheta^k\cdot \left( \mathbf{\hat{V}}^k-\vV_1^{\pi ^k,\mu ^k}\left( s_1 \right) \right)}+\sum_{k=1}^K{\sum_{h=1}^H{\left( \beta _{h}^{k}+\xi _{h}^{k}+\zeta _{h}^{k} \right)}}+\cO\left( \sqrt{dH^2K} \right) 
\\
\overset{\left( iv \right)}{\le}&K\delta +\cO\left( \sqrt{\min\{d,S\}dH^4SABK\iota} \right) 
\end{align*}
where $(i)$ is by $\delta$-approachability, $(ii)$ is by Lemma~\ref{lem:ULCB}, $(iii)$ is by Lemma~\ref{lem:regret-decomposition} and $(iv)$ is by Lemma~\ref{lem:sum-of-bonus} and Azuma-Hoeffding inequality. The claim is proved by taking the union bound with the event that \eqref{equ:sgd-regret} holds.
\end{proof}

\section{Comparison with CMDP literature.} \label{sec:cmdp-compare}
\label{sec:comparsion}

We compare our results with existing works on provably efficient algorithms for CMDP \cite{efroni2020exploration,ding2020provably,brantley2020constrained} in Table~\ref{tab:comparison}. Since the setting is a little bit different in these works, we try to unify the results as below:

\begin{list}{–}{\leftmargin=1.5em}
	\setlength\itemsep{0em}
	\item When measuring constraint violation, \citet{efroni2020exploration, ding2020provably, brantley2020constrained} all consider $\cL^{\infty}$ norm. To compare with our result we have transformed the result to $\cL^2$ norm.
	\item Comparing with the othe algorithm in Table~\ref{tab:comparison}, OptCMDP-bonus actually uses a even stronger notion of regret, by summing up only the non-negative part of the constrain violation in each coordinate. \citet{efroni2020exploration} also propose two more algorithm, OptCMDP and OptDual-CMDP, whose theoretical guarantee is similar to the ones we present and are thus ommited. 
	\item OptPrimalDual-CMDP and OPDOP need an upper bound $\rho$ of the dual variable by assuming the target set is strictly achievable. We use the geometrical assumption (Assumpion~\ref{assump:angle}) instead because the constraints measured by distance cannot be ``strictly'' satisfied (the distance function cannot be below zero). The dependence on $d$ and $\rho$ is not written explicitly in~\citep{ding2020provably}.
	\item \citet{ding2020provably} also considers linear approximation setting and we translate their result to tabular setting. \citet{brantley2020constrained} also considers knapsack setting.
	\item \citet{qiu2020upper} consider MDPs with adversarial reward functions and linear constraints. They assume $\cS_{i} \cap \cS_j = \emptyset$ for $i\neq j$. Therefore, if $|\cS_2| = \cdots = |\cS_H| = S$ then $|\cS| = (H - 1) S + 2 = \cO(HS)$, according to which we translate their regret bound to $\tlO(\sqrt{H^5 S^2 A K})$ in our setting.
	\item \cite{qiu2020upper} and \cite{ding2020provably} also considers adversarial reward but requires full information feedback then. Notice to handle adversarial transition kernels, we still need to game-theoretical formulation in the previous sections.
\end{list}

\section{Proofs for Section~\ref{sec:MDP_upper}}

Besides the notations we introduce at the beginning of Appendix~\ref{sec:basic_proof}, we also set the empirical and population variance operator by
$$
\widehat{\V}^k_h[V](s,a) :=\Var_{s'\sim \Phat^k_h(\cdot|s,a)}V(s'), \,\,\,\,\, \V_h[V](s,a) :=\Var_{s'\sim \P_h(\cdot|s,a)}V(s')
$$
for any function $V \in [-\sqrt{d}H,\sqrt{d}H]^{S}$.

As a result, the bonus terms can be written as
\begin{equation}
\beta := C\big(\sqrt{\frac{\hat{\V}_h \low{V}_{h+1}(s,a)\min\{d,S\}\iota}{N_h^k(s,a)}} +\frac{1}{H} \hat{\P}_{h}(\up{V}_{h+1}-\low{V}_{h+1})(s,a) +\frac{ \min\{d,S\}\sqrt{d}H^2\iota}{N_h^k(s,a)}\big)
\end{equation}
for some absolute constant $C>0$, which is \emph{different} from the one we used in Appendix~\ref{sec:basic_proof}. Another major difference from that Appendix~\ref{sec:basic_proof} is that now we are considering MDP instead of MG.

We still begin with optimism, which is a upper and lower version of Lemma~\ref{lem:ULCB}:
\begin{lemma}\label{lem:optimism_Bernstein}
	With probability $1-p$, for all $h,s,a$ and $k\in[K]$, we have
	\begin{equation}
	\up{Q}^{k}_h(s,a) \ge Q^{\star}_h(s,a) \ge  \low{Q}^{k}_h(s,a), \,\,\,\,\,\, \up{V}^{k}_h(s) \ge V^{\star}_h(s) \ge \low{V}^{k}_h(s).
	\end{equation}
	
\end{lemma}
\begin{proof}
The proof is very similar to that of Lemma~\ref{lem:ULCB}. We only need to bound the variance by induction hypothesis,
\begin{equation*}
		\begin{aligned}
			&|\hat{\V}_h^k \low{V}^k_{h+1} - \hat{\V}_h^kV^{\star}_{h+1}|(s,a)\\
\le & |[\hat{\P}_h^k \low{V}^k_{h+1}]^2-(\hat{\P}_h^kV^{\star}_{h+1})^2|(s,a)+|\hat{\P}_h^k (\low{V}^k_{h+1})^2-\hat{\P}_h^k(V^{\star}_{h+1})^2|(s,a)
			\\
			  \le&
			4\sqrt{d}H\hat{\P}_h^k | V^{\star}_{h+1}-\low{V}^k_{h+1} |(s,a)\\
			 \le &
			4\sqrt{d}H\hat{\P}_h^k (\up{V}^{k}_{h+1} - \low{V}^{k}_{h+1})(s,a).
		\end{aligned}
      \end{equation*}
      
As a result,
		\begin{equation*}
		\begin{aligned}
			\sqrt{\frac{\min\{d,S\}\iota\hat{\V}_h^k V^{\star}_{h+1}(s,a) }{N_h^k(s,a)}}  &\le 
		\sqrt{\frac{\min\{d,S\}\iota\hat{\V}_h^k\low{V}^k_{h+1} + 4\iota \min\{d,S\}\sqrt{d} H\hat{\P}_h^k (\up{V}^{k}_{h+1} - \low{V}^{k}_{h+1})](s,a)}{N_h^k(s,a)}}\\
		&	\le    \sqrt{\frac{\min\{d,S\}\iota\hat{\V}_h^k \low{V}^k_{h+1} (s,a) }{N_h^k(s,a)}}+ \sqrt{\frac{4\iota  \min\{d,S\}\sqrt{d} H\widehat{\P}_h^k (\up{V}^{k}_{h+1} - \low{V}^{k}_{h+1})](s,a) }{N_h^k(s,a)}}\\
	&	\overset{\left( i \right)}{\le} \sqrt{\frac{\min\{d,S\}\iota\hat{\V}_h^k \low{V}^k_{h+1} (s,a) }{N_h^k(s,a)}}
		+ \frac{\hat{\P}_h^k (\up{V}^{k}_{h+1} - \low{V}^{k}_{h+1})(s,a)}{H}
		+ \frac{4 \min\{d,S\}\sqrt{d}H^2\iota}{N_h^k(s,a)}.
		\end{aligned}
		\end{equation*}
		where $(i)$ is by AM-GM inequality.

\end{proof}

Since we are estimating the deviation in exploration bonus using the empirical variance estiamtor, we need to prove it is actually close to the population variance estiamtor. This is true if the corresponding state-action pair has been visited frequently.
\begin{lemma}
	\label{lem:bound_variance}
	Consider a fixed $(s,a)$ triple at $h$-th step. With probability $1-p$,
	$$
	| \hat{\V}_h^k\low{V}^k_{h+1}) - \V_h\V^{\pi^k}_{h+1}|(s,a) \le  4\sqrt{d}H\hat{\P}_h^k(\up{V}^{k}_{h+1} - \low{V}^{k}_{h+1})(s,a) + \cO(1 + \frac{d^2H^4S\iota}{N_h^k(s,a)}).
	$$
\end{lemma}

\begin{proof}
   Following the same argument in Lemma~\ref{lem:optimism_Bernstein}, we have $\up{V}^{k}_h(s) \ge V_h^{\pi^k}(s) \ge \low{V}^{k}_h(s)$. As a result,
      \begin{align*}
      &	 |\hat{\V}_h^k \low{V}^k_{h+1}) - \V_hV^{\pi^k}_{h+1}|(s,a) \\
         = & | [\hat{\P}_h^k(\low{V}^k_{h+1})^2 - \P_h(V^{\pi^k}_{h+1})^2](s,a)
   - [(\hat{\P}_h^k (\low{V}^k_{h+1}))^2 - (\P_hV^{\pi^k}_{h+1})^2](s,a)| \\
   \le & [\hat{\P}_h^k( \up{V}^k_{h+1})^2 - \P_h(\low{V}^k_{h+1})^2
   - (\hat{\P}_h^k \low{V}^k_{h+1})^2 + (\P_h\up{V}^k_{h+1})^2](s,a)
   \\
   \le &[|(\hat{\P}_h^k-\P_h)( \up{V}^k_{h+1})^2|+|\P_h[( \up{V}^k_{h+1})^2-(\low{V}^k_{h+1})^2]|+|(\hat{\P}_h^k \low{V}^k_{h+1})^2-(\P_h \low{V}^k_{h+1})^2|+|(\P_h \low{V}^k_{h+1})^2-(\P_h\up{V}^k_{h+1})^2|](s,a)
      \end{align*}
   These terms can be bounded separately by
   \begin{align*}
      |(\hat{\P}_h^k-\P_h)( \up{V}^k_{h+1})^2|(s,a) &\le \cO(dH^2\sqrt{\frac{S\iota}{N_h^k(s,a)}}),
      \\
      |\P_h[( \up{V}^k_{h+1})^2-(\low{V}^k_{h+1})^2]|(s,a,b) &\le 2\sqrt{d}H[\P_h( \up{V}^k_{h+1}-\low{V}^k_{h+1})](s,a),
      \\
      |(\hat{\P}_h^k \low{V}^k_{h+1})^2-(\P_h \low{V}^k_{h+1})^2|(s,a,b) &\le 2\sqrt{d}H[(\hat{\P}_h^k-\P_h)\low{V}^k_{h+1}](s,a) \le \cO(dH^2\sqrt{\frac{S\iota}{N_h^k(s,a)}}),
      \\
      |(\P_h \low{V}^k_{h+1})^2-(\P_h\up{V}^k_{h+1})^2|(s,a) &\le 2\sqrt{d}H[\P_h( \up{V}^k_{h+1}-\low{V}^k_{h+1})](s,a).	
   \end{align*}
   
   Combining with $dH^2\sqrt{\frac{S\iota}{N_h^k(s,a)}} \le 1 + \frac{d^2H^4S\iota}{N_h^k(s,a)}$ completes the proof.
   \end{proof} 

   The last auxiliary lemma is borrowed from \cite{liu2020sharp} to handle the $\hat{\P}_h^k(\up{V}^{k}_{h+1} - \low{V}^{k}_{h+1})(s,a)$ term. For completeness we give a proof here due to difference in setting.

   \begin{lemma}
      \label{lem:lower-order}
      For any function $V \in [0, H]^{\cS}$ s.t. $ |V|(s) \le (\up{V}^{k}_{h+1} - \low{V}^{k}_{h+1})(s) $ for any $s$, with probability $1-p$,
      \begin{align*}
      |(\hat{\P}_h^k- \P_h)V(s,a)| \le  \cO\bigg(\frac{1}{H}\min\{\hat{\P}_h^k (\up{V}^{k}_{h+1} - \low{V}^{k}_{h+1})(s,a),\P_h (\up{V}^{k}_{h+1} - \low{V}^{k}_{h+1})(s,a)\}
      + \frac{H^2S\iota}{N_h^k(s,a)}\bigg).
      \end{align*}
   \end{lemma}
   \begin{proof}
      By triangle inequality,
       \begin{align*}
         |(\hat{\P}_h^k- \P_h)V(s,a)|
         \le& \sum_{s'}{|(\hat{\P}_h^k- \P_h)(s'|s,a,b)||V|(s')}\\
       \le& \sum_{s'}{|(\hat{\P}_h^k- \P_h)(s'|s,a)|(\up{V}^{k}_{h+1} - \low{V}^{k}_{h+1})(s')}\\
       \overset{\left( i \right)}{\le}& \cO\left(\sum_{s'}{(\sqrt{\frac{\iota \hat{\P}_h^k(s'|s,a)}{N_h^k(s,a)}}+\frac{\iota }{N_h^k(s,a)})(\up{V}^{k}_{h+1} - \low{V}^{k}_{h+1})(s')}\right)\\
       \overset{\left( ii \right)}{\le}& \cO\left(\sum_{s'}{(\frac{\hat{\P}_h^k(s'|s,a) }{H}+\frac{H\iota }{N_h^k(s,a)})(\up{V}^{k}_{h+1} - \low{V}^{k}_{h+1})(s')}\right)\\
       \le& \cO\left(\frac{\hat{\P}_h^k (\up{V}^{k}_{h+1} - \low{V}^{k}_{h+1})(s,a)}{H}
      + \frac{H^2S\iota}{N_h^k(s,a)}\right),
       \end{align*}
       where $(i)$ is by empirical Bernstein bound \citep{maurer2009empirical} and $(ii)$ is by AM-GM inequality. This proves the empirical version. Use the standard Bernstein bound, we get the a similar upper bound. Combining the two bounds completes the proof.
   \end{proof}

Combining the previous results we can prove a tighter version of Lemma~\ref{lem:sum-of-bonus}, which is the key lemma in the proof of Theorem~\ref{thm:bernstein}.

\begin{lemma}[Sum of bonus]
   \label{lem:sum-of-bonus-Bernstein}
   $$
   \sum_{k=1}^K{\sum_{h=1}^H{\beta _{h}^{k}}}\le O\left( \sqrt{\min\{d,S\}dH^3SABK\iota} \right) 
$$
\end{lemma}
\begin{proof}
   Define $\Delta_h^k:= [\up{V}_{h}^{k}-\low{V}_{h}^{k}](s_{h}^{k})$. Then 
   \begin{align}
      \label{equ:Bernstein_decomposition}
      \Delta_h^k =\left( \up{Q}^k_{h}-\low{Q}_{h}^{k} \right) (s_{h}^{k},a_{h}^{k})  \le \P_h\left( \up{V}^k_{h+1}-\low{V}_{h+1}^k \right) \left( s_{h}^{k},a_{h}^{k} \right) +\beta _{h}^{k}
   =\Delta_{h+1}^k +\beta _{h}^{k}+\zeta _{h}^{k}.
   \end{align}
   where 
   \begin{align*}
      \zeta _{h}^{k}:=&\P_h\left( \up{V}_{h+1}-\low{V}_{h+1}^k \right) \left( s_{h}^{k},a_{h}^{k} \right) -\left( \up{V}_{h+1}^k-\low{V}_{h+1}^k \right) \left( s_{h}^{k} \right) \in \left[ -4\sqrt{d}H,4\sqrt{d}H \right]  .
   \end{align*}
   
   We only need to carefully bound $\beta _{h}^{k}$. 
   \begin{align*}
      \beta _{h}^{k} = \cO \left(\sqrt{\frac{\hat{\V}_h \low{V}^k_{h+1}(s_{h}^{k},a_{h}^{k})\min\{d,S\}\iota}{N_h^k(s_{h}^{k},a_{h}^{k})}} + \frac{1}{H}\hat{\P}_{h}(\up{V}^k_{h+1}-\low{V}^k_{h+1})(s_{h}^{k},a_{h}^{k}) + \frac{\min\{d,S\}\sqrt{d}H^2\iota}{N_h^k(s_{h}^{k},a_{h}^{k})}\right)
   \end{align*}
   
   By Lemma~\ref{lem:bound_variance} and AM-GM inequality,
   \begin{equation}
      \begin{aligned}
         &\sqrt{\frac{\min\{d,S\}\iota\hat{\V}_h^k \low{V}^k_{h+1}(s,a)}{N_h^k(s,a)}} \\
         \le&
         \sqrt{\frac{\min\{d,S\}\iota\V_h V^{\pi^k}_{h+1}(s,a) + \cO(\min\{d,S\}\iota)}{N_h^k(s,a)}} + \sqrt{\frac{4\min\{d,S\}\sqrt{d}H\iota\hat{\P}_h^k (\up{V}^{k}_{h+1} - \low{V}^{k}_{h+1})(s,a,b)}{N_h^k(s,a)}}+ \cO(\frac{dH^2\sqrt{S}\iota}{N_h^k(s,a)})\\
         \le& \sqrt{\frac{\min\{d,S\}\iota\V_h\V^{\pi^k}_{h+1}(s,a) + \cO(\min\{d,S\}\iota)}{N_h^k(s,a)}} + 
         \frac{1}{H}\hat{\P}_h^k (\up{V}^{k}_{h+1} - \low{V}^{k}_{h+1})(s,a)
      + \cO(\frac{\min\{d,S\}\sqrt{d}H^2\sqrt{S}\iota}{N_h^k(s,a)}).
      \end{aligned}
      \end{equation}

      Using Lemma~\ref{lem:lower-order} we have
      \begin{equation*}
      \hat{\P}_h^k (\up{V}^{k}_{h+1} - \low{V}^{k}_{h+1})(s,a) \le \paren{1 + \frac{\cO(1)}{H}}\P_h (\up{V}^{k}_{h+1} - \low{V}^{k}_{h+1})(s,a)+\cO(\frac{H^2S\iota}{N_h^k(s,a)}).
   \end{equation*}

      Plugging back into inequality~\eqref{equ:Bernstein_decomposition} we have 
      \begin{equation}
      \begin{aligned}
         \Delta_h^k\le \paren{1 + \frac{\cO(1)}{H}}\Bigg\{\Delta_{h+1}^k+ \zeta_h^k+\cO\bigg(\sqrt{\frac{\min\{d,S\}\iota\V_hV^{\pi^k}_{h+1}(s_h^k,a_h^k)}{N_h^k(s_h^k,a_h^k)}} + \sqrt{\frac{\min\{d,S\}\iota}{N_h^k(s_h^k,a_h^k)}}	+ \frac{\min\{d,S\}\sqrt{d}H^2S\iota}{N_h^k(s_h^k,a_h^k)}\bigg)\Bigg\}.
         \end{aligned}
      \end{equation}
   
      Recursing this argument for $h \in [H]$ and taking the sum,
   
   \begin{align*}
   \sum_{k=1}^{K}{\Delta_1^k} \le  \sum_{k=1}^{K}\sum_{h=1}^{H}{\cO\left(\zeta_h^k +\sqrt{\frac{\min\{d,S\}\iota\V_h V^{\pi^k}_{h+1}(s_h^k,a_h^k)}{N_h^k(s_h^k,a_h^k)}} + \sqrt{\frac{\min\{d,S\}\iota}{N_h^k(s_h^k,a_h^k)}} 	+ \frac{\min\{d,S\}\sqrt{d}H^2S\iota}{N_h^k(s_h^k,a_h^k)}\right)} .
   \end{align*}
   
   The remaining steps are exactly the same as that in the proof of Theorem~\ref{thm:basic-approachable}. The only difference is that we need to bound the sum of variance term by Cauchy-Schwarz,
   \begin{equation}
      \label{eq:sumlog}
   \sum_{k=1}^{K}\sum_{h=1}^{H}{\frac{1}{N_h^k(s_h^k,a_h^k)}} \le \sum_{s,a,h}\sum_{n=1}^{N_h^k(s,a)}{\frac{1}{n}}\le \cO \paren{HSA\iota}.
   \end{equation}
   and
   \begin{align*}
      \sum_{k=1}^{K}\sum_{h=1}^{H}{\sqrt{\frac{\V_h V^{\pi^k}_{h+1}(s_h^k,a_h^k)}{N_h^k(s_h^k,a_h^k)}}}\le & \sqrt{\sum_{k=1}^{K}\sum_{h=1}^{H}{\V_h V^{\pi^k}_{h+1}(s_h^k,a_h^k)}\cdot \sum_{k=1}^{K}\sum_{h=1}^{H}{\frac{1}{N_h^k(s_h^k,a_h^k)}}}\\
      \overset{\left( i \right)}{\le} & \cO \paren{\sqrt{d(H^2K+H^3\iota)\cdot HSA\iota}}\\
      =&\cO \paren{\sqrt{dH^3SAK}+\sqrt{dH^4SA\iota^2}},
   \end{align*}
   where $(i)$ is by Law of total variation (for example, Lemma 8 in \citet{azar2017minimax}) and inequality~\eqref{eq:sumlog}.   
\end{proof}

 \begin{proof}[Proof of Theorem~\ref{thm:bernstein}]
We consider the two possibilities separately.

\textbf{Using \PDU{} for \DU{}.} As in the proof of Theorem~\ref{thm:basic-approachable} we have 
$$
K\dist\left( \vW^K,W^* \right) ^2\le dH^2+2\sum_{k=1}^K{\frac{k-1}{K}}\dist\left( \vW^{k-1},W^{\star} \right) \left[\delta+ \sum_{h=1}^H{\left( \beta _{h}^{k}+\xi _{h}^{k}+\zeta _{h}^{k} \right)}+\left(\vtheta^k\cdot \mathbf{\hat{V}}^k- V_{1}^{\mu ^k,\nu ^k}\left( s_{1}^{k} \right)  \right) \right] 
$$

Again we prove the theorem by induction. Suppose $$
\dist \left( \vW^k,W^{\star} \right) \le \delta+ c_0 \sqrt{\min\{d,S\}dH^3SAB\iota/k}
$$ for $\forall k \le K-1$, let's prove the claim holds for $k=K$. Now we have a new term to bound, which is
$$
2\delta \sum_{k=1}^K{\frac{k-1}{K}}\left[ \delta +c_0\sqrt{\min \{d,S\}dH^3SAB\iota /k}+\sum_{h=1}^H{\left( \beta _{h}^{k}+\xi _{h}^{k}+\zeta _{h}^{k} \right)}+\left( \boldsymbol{\theta }^k\cdot \mathbf{\hat{V}}^k-V_{1}^{\mu ^k,\nu ^k}\left( s_{1}^{k} \right) \right) \right] =\left( A \right) +\left( B \right) +\left( C \right) 
$$
where 
$$
\left( A \right) =2\delta ^2\sum_{k=1}^K{\frac{k-1}{K}}\le \left( K-1 \right) \delta ^2,
$$
$$
\left( B \right) =2\delta \sum_{k=1}^K{\frac{k-1}{K}}c_0\sqrt{\min \{d,S\}dH^3SAB\iota /k}\le \frac{4}{3}c_0\delta \sqrt{\min \{d,S\}dH^3SAB\iota K},
$$
and by Lemma~\ref{lem:sum-of-bonus-Bernstein} and Azuma-Hoeffding inequality,
$$
\left( C \right) =2\delta \sum_{k=1}^K{\frac{k-1}{K}}\left[ \sum_{h=1}^H{\left( \beta _{h}^{k}+\xi _{h}^{k}+\zeta _{h}^{k} \right)}+\left( \boldsymbol{\theta }^k\cdot \mathbf{\hat{V}}^k-V_{1}^{\mu ^k,\nu ^k}\left( s_{1}^{k} \right) \right) \right] \le c_1c_0\delta \sqrt{\min \{d,S\}dH^3SAB\iota K}.
$$

To prove the induction hypothesis, we only need to guarantee
\begin{align*}
   &H^2+8c_0\max \left\{ c_1,c_2 \right\} \min \{d,S\}dH^3SAB\iota +\left( K-1 \right) \delta ^2+c_1\delta \sqrt{\min \{d,S\}dH^3SAB\iota K}
\\
\le& K\left[ \delta +c_0\sqrt{\min \{d,S\}dH^3SAB\iota /K} \right] ^2.
\end{align*}

Comparing the coefficients, we can see this is satisfied by setting  $c_0\ge \max \left\{ c_1,\sqrt{\frac{2}{SABH\iota}} \right\}$ .

\textbf{Using \ODU{} for \DU{}.} We can expand the distance using the same argument in the proof of Theorem~\ref{thm:OCO-delta-approachable}

\begin{align*}
K\dist\left( \vW^k,W^{\star} \right) \le&K\delta +\sum_{k=1}^K{\vtheta^k\cdot \left( \mathbf{\hat{V}}^k-\vV_1^{\pi ^k,\mu ^k}\left( s_1 \right) \right)}+\sum_{k=1}^K{\sum_{h=1}^H{\left( \beta _{h}^{k}+\xi _{h}^{k}+\zeta _{h}^{k} \right)}}+\cO\left( \sqrt{dH^2K} \right) 
\\
\overset{\left( i \right)}{\le}&K\delta +\cO\left( \sqrt{\min\{d,S\}dH^3SABK\iota} \right) 
\end{align*}
where $(i)$ by Lemma~\ref{lem:sum-of-bonus-Bernstein} and Azuma-Hoeffding inequality. The claim is proved by taking the union bound.
 \end{proof}

 \section{Proofs for Section~\ref{sec:satisfiable}}
\label{sec:opt}

\begin{proof}[Proof of Theorem~\ref{thm:satisfiable}]

    A crucial property we will use is, by the definition of fenchel duality, 
\begin{equation}
   \label{equ:fenchel-property}
   g^*\left( \vtheta \right)=\underset{\vx\in X}{\max}\left\{ \vtheta\cdot \vx-g\left( \vx \right) \right\} \ge \vtheta\cdot \vV^{\mu^{\star}}_1(s_1)-g\left( \vV^{\mu^{\star}}_1(s_1) \right) 
\end{equation}
and
\begin{equation}
   \label{equ:satisfiable}
   \underset{\vx\in W^{\star}}{\max}\vtheta\cdot \vx \ge \vtheta\cdot \vV^{\mu^{\star}}_1(s_1)
\end{equation}
for any $\nu$.

Let's try to bound the regret and constraint violation.

\begin{align*}
  & K\left[g(\vW^K)- g(\vV^{\mu^{\star}}_1(s_1))+\rho \dist\left( \vW^K,W^{\star} \right) \right] \\
  =&\underset{\left\| \bm{\phi  } \right\| \le 1}{\max}\left\{ \bm{\phi  }\cdot \sum_{k=1}^K{\mathbf{\hat{V}}^k}- \sum_{k=1}^Kg^*\left( \bm{\phi  } \right) \right\}-K  g(\vV^{\mu^{\star}}_1(s_1))+\rho \underset{\left\| \bm{\varphi  } \right\| \le 1}{\max}\left\{ \bm{\varphi  }\cdot \sum_{k=1}^K{\mathbf{\hat{V}}^k}- \sum_{k=1}^K \underset{\vx\in W^{\star}}{\max}\bm{\varphi}\cdot \vx \right\}  \\
  \le &\sum_{k=1}^K{\left\{ \bm{\phi}^k\cdot \mathbf{\hat{V}}^k-g^*\left( \bm{\phi^k  } \right) \right\}}- K  g(\vV^{\mu^{\star}}_1(s_1)) +\rho \sum_{k=1}^K{\left\{ \bm{\varphi}^k\cdot \mathbf{\hat{V}}^k-\underset{\vx\in W^{\star}}{\max}\bm{\varphi}^k\cdot \vx \right\}}+\cO\left( \rho \sqrt{dH^2K} \right) \\
  \overset{\left( i \right)}{\le}&\sum_{k=1}^K\left[ \vtheta^k \cdot (\mathbf{\hat{V}}^k -\vV^{\mu^{\star}}_1(s_1) )\right]+\cO\left( \rho \sqrt{dH^2K} \right) \\
  \overset{\left( ii \right)}{\le}&\sum_{k=1}^K\left[ \vtheta^k \cdot (\mathbf{\hat{V}}^k -\vV^{\mu^{k}}_1(s_1) )\right]+\sum_{k=1}^K\left[ \vV^{\mu^{k}}_1(s_1)- V^{k}_1(s_1) \right]+\cO\left( \rho \sqrt{dH^2K} \right) \\
  \overset{\left( iii \right)}{\le}&\cO\left( \rho \sqrt{\min\{d,S\}dH^3SAK\iota} \right), 
\end{align*}
where $(i)$ is by the update \DODU{} and inequality~\eqref{equ:fenchel-property}~\eqref{equ:satisfiable},  $(ii)$ is by optimism and $(iii)$ is by Lemma~\ref{lem:sum-of-bonus}.

\paragraph{Bound constraint violation in constrained MDP}

We need to define a few notations. Recall that  return vectors $\vr_h(s, a)$ live in a space $\mathbb{R}^d$, and  $\W^{\star} \subseteq \mathbb{R}^d$ denotes the set of desired expected future return. 

Note that $ \dist\left( \vW^K,W^{\star} \right)  \ge 0$; hence we have $K \left[g(\vW^K)-  g(\vV^{\mu^{\star}}_1(s_1)) \right] = \cO\left( \rho \sqrt{\min\{d,S\}dH^3SAK\iota} \right)$. To bound the constraint violation separately, we only need the lemma below.

\begin{figure}
	\centering     
	\subfigure[]{\label{fig:a}\includegraphics[width=60mm]{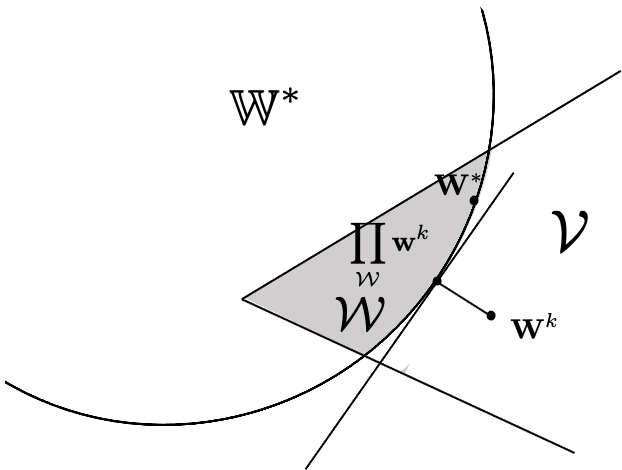}}
	\subfigure[]{\label{fig:b}\includegraphics[width=60mm]{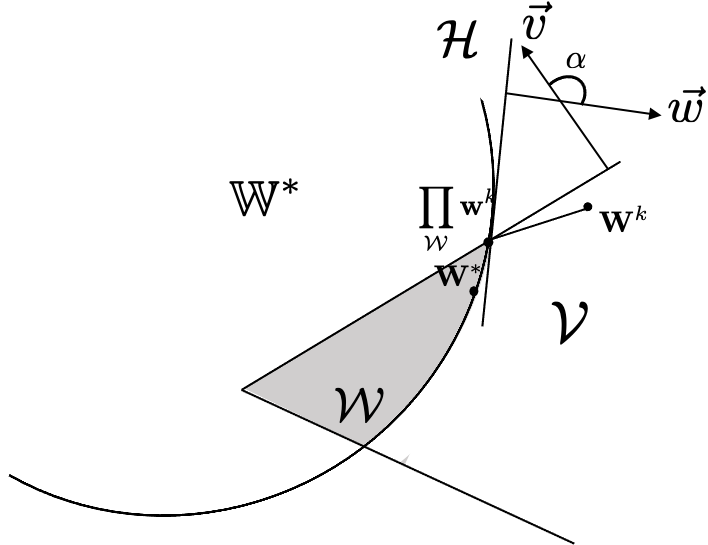}}
	\caption{	 The projected point is : (a)In the interior ; (b)on the boundary.}
    \label{fig:proof}
\end{figure}

\begin{lemma}
Let $\vW^{\star}$ denote a return vector in set $\cW$ that achieves the lowest cost,  i.e. $\forall \  \vW \in \cW, g(\vW) \ge g(\vW^{\star})$,   then under Assumption~\ref{assump:angle},
\begin{align*}
	\left[g(\vW^K)-   g( \vW^{\star} ) \right]   \ge  -  \dist\left( \vW^K, \W^{\star} \right)  /\gamma_{\min}.
\end{align*}
\end{lemma}

\begin{proof}

  Note that if $\vW^K \in \cW$, then by optimality of $\vW^{\star}$,
  \begin{align*}
	\left[g(\vW^K)-   g( \vW^{\star} ) \right]   \ge   0 \ge -  \dist\left( \vW^K, \W^{\star} \right)  /\gamma_{\min}.
  \end{align*}

  We focus on the case when $\vW^K  \not\in \cW$.  By convexity of $\cW$, there exists a unique $\prod_\cW \vW^K =  \argmin_{\vW \in \cW} \dist(\vW, \vW^k)$.  Again we study two cases as illustrated in Figure~\ref{fig:proof}: whether or not the projected point $\prod_{\cW} \vW^K$ is in the interior of $\cV$. 
  
  \paragraph{Case 1:in the  interior.} Note that the projection can be described as an optimization operation: $\prod_\cW \vW^K  = \argmin_{w \in \cV, w \in \W^{\star}} \dist(\vW^K,  w)$. When the projected point is in the interior of $\cV$, we know that the constraint $w \in \cV$ is not active at the optimal solution. Hence by complementary slackness, $\dist\left( \vW^K, \cW \right) = \min_{\vW \in \cV, \vW  \in \W^{\star}} \dist(\vW^K,  \vW ) = \min_{\vW  \in \W^{\star}} \dist(\vW^K,  \vW ) = \dist\left( \vW^K, \W^{\star} \right) $. Then the inequality simply follows by    
  \begin{align*}
  	\left[g(\vW^K)-   g( \vW^{\star} ) \right]   =& g(\vW^K)-  g(\prod_\cW \vW^K)  + g(\prod_\cW \vW^K)  -  g( \vW^{\star} ) \\
  	\ge&  - \dist(\vW^K, \cW) +  0 \ge  - \dist(\vW^K, \W^{\star}) 
  \end{align*}

  \paragraph{Case 2:on the boundary.} In this case, the distance $\dist\left( \vW^K, \cW \right)$ may not equal $\dist\left( \vW^K, \W^{\star} \right) $.   Instead, we know by convexity and Assumption~\ref{assump:angle} that the support hyperplanes of  $\cV$ and $\W^{\star}$ at $\prod_\cW \vW^K $ intersects with an angle $\alpha(\prod_\cW \vW^K) < \pi $, where $\alpha$ is defined in Assumption~\ref{assump:angle}.   By optimality of the $\prod_\cW \vW^K$ in solving $ \min_{\vW \in \cV, \vW  \in \W^{\star}} \dist(\vW^K,  \vW ) $, the vector $\prod_\cW \vW^K  \to  \vW^K$ must lie in the cone formed by the support vectors. Further by $\vW^K \in \cV$, we have $\alpha(\prod_\cW \vW^K) \ge \pi/2 $.
  
  Then we know that
  \begin{align*}
    \left[g(\vW^K)-   g( \vW^{\star} ) \right]   =& g(\vW^K)-  g(\prod_\cW \vW^K)  + g(\prod_\cW \vW^K)  -  g( \vW^{\star} ) \\
    \ge&  - \dist(\vW^K, \prod_\cW \vW^K) +  0
  \end{align*}
  where the second line follows by the Lipschitzness of $g$. Denote $\mathcal{H} = \mathcal{H}(\prod_\cW \vW^K)$ as the hyperspace that is supported by  the support vector of $W^{\star}$ at $\prod_\cW \vW^K$.  Then by the fact that $\mathcal{W^{\star}} \subseteq \mathcal{H} $ and assumption \ref{assump:angle}, we get
  \begin{align}
    \dist(\vW^K, \W^{\star}) \ge \dist(\vW^K, \mathcal{H}) \ge \dist(\vW^K,  \prod_\cW \vW^K) \sin(\pi - \alpha(\prod_\cW \vW^K))
  \end{align}
  Rearrange and we get
  \begin{align*}
    \left[g(\vW^K)-   g( \vW^{\star} ) \right]   \ge  -  \dist\left( \vW^K,\W^{\star} \right)  /\gamma_{\min}.
  \end{align*}

\end{proof}

With the above lemma, we see that if $\rho = 2 / \gamma_{\min}$, then 

\begin{align*}
  \frac{1}{ \gamma_{\min}}  K \dist\left( \vW^K,W^{\star} \right)  &\le K\left[g(\vW^K)-  g(\vV^{\mu^{\star}}_1(s_1))+ \frac{2}{ \gamma_{\min}} \dist\left( \vW^K,\W^{\star} \right) \right] \\
  & = K\left[g(\vW^K)-  g(\vV^{\mu^{\star}}_1(s_1))+ \rho \dist\left( \vW^K,\W^{\star} \right) \right]
  \le \cO\left( \rho\sqrt{\min\{d,S\}dH^3SAK\iota} \right).
\end{align*}

Divide both side by $\rho/2$ and we get the desired result.

\end{proof}

\subsection{Necessity of nonsingular intersection}\label{sec:intersect-proof}

  \begin{figure}[t]
	\centering
	\includegraphics[width=0.5\textwidth]{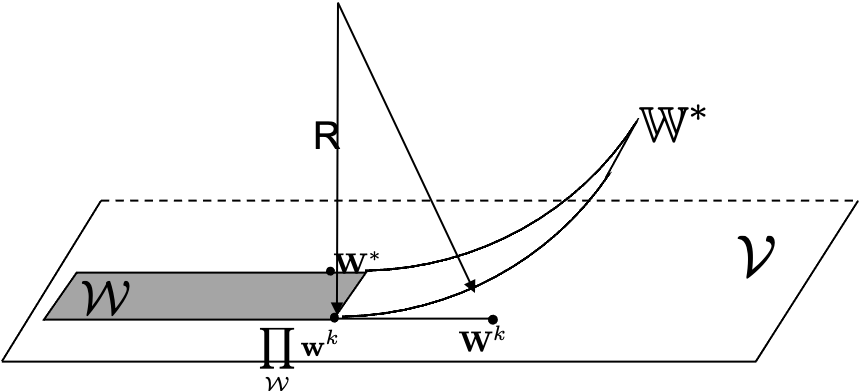}
	\caption{If the support hyperplains intersect with angle $0$ (exact cut), then the point $\vW^k$ can be arbitrarily close to the set $\W^\star$ while remaining away from $\cW$.}\label{fig:intersection2}
\end{figure}

In this section, we explain the high level intuition of why the intersection between the constrain set $\W^{\star}$ and the feasible return vectors $\cV$ needs to be nonsingular. The key problem arises from the fact that $\W^\star$ is defined in space $\mathbb{R}^d$ where the set of feasible return vectors $\cV$ may not be of full dimension. In such cases, for the actual achievable constrain set of interest $\cW = \W^{\star} \cap \cV$, there are too much freedom in selecting $\W$ as long as its elements on $\cV$ remains fixed. In particular,  an achievable return vector $\vW^K$ can be very far away from the constrained feasible set $\cW$, where as being arbitrarily close to the set $\W^\star$. The process is illustrated in Figure~\ref{fig:intersection2} by sending the radius $R$ to infinity. Note that $\cW$ remains unchanged in this process. Since the cost is measured by the distance to $\W^\star$ instead of to the actual set of interest $\cW$, the deviation from $\vW$ to $\cW$ cannot be reduced to 0 with any fixed algorithm given that point $\vW$ is already very close to the target set $\W^{\star}$.

Quantifying the level of non-singularity is necessary, whereas lower bounding the angle at intersection is one natural way of many to do so.

\end{document}